\documentclass[11pt]{article}

\usepackage[utf8]{inputenc}
\usepackage[T1]{fontenc}
\usepackage[margin=1in]{geometry}
\usepackage{graphicx}
\usepackage{url}
\usepackage{booktabs}
\usepackage{amsfonts}
\usepackage{nicefrac}
\usepackage[expansion=false]{microtype}
\usepackage{xcolor}
\usepackage{amsmath, bm, bbm}
\usepackage{amsthm}
\usepackage[normalem]{ulem}
\usepackage{algorithm,algorithmicx,algpseudocode,array,comment}
    {
      \theoremstyle{plain}
      
      \newtheorem{lemma}{Lemma}
      \newtheorem{corollary}{Corollary}
      \newtheorem{theorem}{Theorem}
    }
\theoremstyle{remark}
\newtheorem{remark}{Remark}
\theoremstyle{plain}
\usepackage{placeins}
\usepackage{tikz}
\usetikzlibrary{positioning}
\usepackage{natbib}
\usepackage{hyperref}
\hypersetup{
  colorlinks=true,
  linkcolor=blue,
  citecolor=blue,
  urlcolor=blue
}

\begin{document}

\title{Nuclear Norm-Regularized Bayesian Matrix Completion}

\author{Calvin Tolbert\\
    School of Operations Research and Information Engineering \\
    Cornell University \\
    \texttt{cmt238@cornell.edu}}

\date{October 23, 2026}

\maketitle

\begin{abstract}
Nuclear-norm-regularized least squares is a standard approach to matrix completion, but it yields only a point estimate. Its Bayesian counterpart, a prior combining the nuclear norm with a Frobenius-norm term, gives a strongly log-concave though nonsmooth posterior when the noise precision is known. In practice the precision is unknown, and placing a truncated Jeffreys prior on it makes the joint posterior neither log-concave nor smooth. To our knowledge, we give the first sampler for this model with an explicit non-asymptotic guarantee: its output is within a prescribed total-variation or Wasserstein distance of the posterior on the matrix after a number of steps polynomial in the matrix dimensions and the reciprocal of the target accuracy. The sampler never targets a continuous distribution that is not strongly log-concave. It discretizes the precision onto a grid, estimates its posterior by thermodynamic integration along the tilted conditional posteriors of the matrix, and bounds the error this propagates to the posterior on the matrix. The approach applies whenever non-log-concavity is confined to a bounded scalar hyperparameter and the conditional distributions are strongly log-concave with smooth-plus-Lipschitz potentials. Our result establishes feasibility; the polynomial complexity is not intended as a practical algorithm at current problem scales.

\end{abstract}

\section{Introduction}

Nuclear-norm-regularized least squares has become a common approach to matrix completion (\cite{Negabhan2010, athey2021matrix, ChenYuxin2020}), but it yields only a point estimate. A Bayesian treatment would quantify uncertainty, but only if its posterior could be sampled reliably. The obstacle we address is the noise precision. When it is known, a prior combining the nuclear and Frobenius norms gives a strongly log-concave posterior, but this is seldom the case in practice. Existing treatments fix it in advance (\cite{ChenYuxin2020}), choose it by cross-validation (\cite{athey2021matrix}), infer it by variational Bayes under an improper prior (\cite{Babacan2012}), or integrate it out by Gibbs sampling with no non-asymptotic guarantee on the result (\cite{Cui2024}).

We place a truncated Jeffreys prior on the precision and give an algorithm, with explicit non-asymptotic total-variation and Wasserstein bounds, for sampling from the posterior over the low-rank matrix with the precision marginalized out. The joint posterior is then neither log-concave (Lemma~\ref{lem:nonlc}) nor smooth, and even the marginal posterior of the precision need not be log-concave. We use this model as a testbed for a more general question: can a posterior be sampled with polynomial-time guarantees when its non-log-concavity is restricted to a single bounded scalar hyperparameter? The strategy we propose is never to sample from a continuous distribution that is not strongly log-concave.

\paragraph{Contributions.}
\begin{itemize}
    \item \emph{Marginalizing the precision by thermodynamic integration} (Sections~\ref{sec:discretize} and~\ref{sec:posterior}). We discretize the precision onto a grid and estimate its posterior by thermodynamic integration along the tilted conditional posteriors of the matrix (Algorithm~\ref{alg:ti}), each of which is strongly log-concave. Importance-reweighting draws from the untilted prior can require a number of samples exponential in the number of observations (Lemma \ref{lem:is}), and sampling the correctly tilted measure at each grid point removes this dependence. Thermodynamic integration and the related bridge-sampling, annealed-importance-sampling and stepping-stone estimators are already common in the literature (\cite{GelmanMeng1998, MengWong1996, Neal2001, Xie2011}). Our contribution is its non-asymptotic analysis in this setting.
    \item \emph{Propagating the error to the posterior on the matrix} (Theorems~\ref{thm:TVguar} and~\ref{thm:W1guar}). The Kullback--Leibler divergence between conditional posteriors at different precisions is at most quadratic in the difference of the precisions (Lemma~\ref{lem:bregman}), so the error in the precision posterior propagates linearly to the posterior over the matrix, and the total number of steps is polynomial in the matrix dimensions and the reciprocal of the target accuracy (Corollaries~\ref{cor:TV} and~\ref{cor:W1}). Existing non-asymptotic analyses of normalizing-constant estimation, for log-concave targets (\cite{Brosse2018, Ge2020}) and beyond (\cite{Guo2026}), bound the error of the estimate itself rather than that of a posterior sampled downstream of it.
    \item \emph{Numerical validation} (Section~\ref{sec:numerical}). Where the posterior is known exactly, the sampler matches it to Monte Carlo accuracy at budgets far below those of the worst-case bounds.
\end{itemize}

For the conditional samplers, we rely on existing results. Non-asymptotic sampling theory has mostly been developed for smooth targets, and our conditional posteriors have potentials of the form (smooth) $+$ (nonsmooth, Lipschitz). For this class, \cite{Mou22} give an efficient Metropolis--Hastings scheme and the proximal sampler based on a restricted Gaussian oracle (\cite{LeeShenTian2021, ChenChewi2022}) is another route. Both need a sampling oracle for the nonsmooth term, which the nuclear norm does not have in closed form, although the rejection-sampling oracle of \cite{LiangChen2022} is exact for any convex Lipschitz term and, with the analyses of \cite{ChenChewi2022, FanYuanChen2023}, yields polynomial mixing bounds comparable to ours. We instead use random-walk Metropolis, which needs only evaluations of the potential. \cite{LiuTong2026} show that the acceptance-probability bound underlying the isoperimetric-profile analysis of \cite{Andrieu24} extends to smooth-plus-Lipschitz potentials (we include a short proof in our notation, Lemma~\ref{lem:accept}), giving mixing bounds polynomial in the dimension and logarithmic in the target accuracy. This is one sufficient route to the high-accuracy conditional samplers our analysis needs, but not the only one.

\paragraph{Notation.} We observe some entries of an $n_1 \times n_2$ matrix $\bm{Y}$, with $n_1 \leq n_2$ without loss of generality. Unobserved entries are denoted by $\square$, with the convention $0\square = 0$. The sampling mask $\bm{\Omega} \in \{0,1\}^{n_1 \times n_2}$ equals 1 at observed entries and 0 at unobserved ones. $N$ is the number of observed entries, and $\circ$ is the entrywise product, as in $\bm{Y \circ \Omega}$. Observed entries take the form $Y_{ij} = L_{ij} + E_{ij}$, where $\bm{L} \in \mathbb{R}^{n_1 \times n_2}$ is the latent low-rank matrix we want to estimate and the residuals $E_{ij}$ are i.i.d.\ Gaussian with mean 0 and precision (reciprocal of variance) $\tau$.

\section{Background and related work}
The field of matrix completion has a rich history, going back to the seminal papers of \cite{Keshavan2010Main, Keshavan2010Later, candestao10} and the more general paper of \cite{Negabhan2010}.
Applications range from causal inference in panel data as in \cite{bai21, agarwal21causal, athey2021matrix} to recommender systems as in \cite{koren2009matrix}.
The nuclear norm $\|\bm{L}\|_*$ of a matrix $\bm{L}$ is the sum of its singular values, so it penalizes high-rank matrices, thus encouraging simplicity of the latent-factor model while not forcing it to be a specific rank. This allows us to move beyond the synthetic control approach of \cite{Abadie2010} and the fixed-rank approach of \cite{bai03} while still encouraging approximate low-rankness. Some frequentist examples of the nuclear-norm approach are \cite{Negabhan2010, ChenYuxin2020, athey2021matrix}. \cite{Cui2024} introduce a Bayesian approach by imposing a prior on the singular value decomposition of the matrix and develop a Hamiltonian Monte Carlo scheme targeted to this setting. Bayesian matrix factorization sampled by Gibbs sampling is another established baseline (\cite{SalakhutdinovMnih2008}), and Bayesian approaches to matrix completion with statistical guarantees also exist (e.g., \cite{MaiAlquier2015}).
 
Bayesian interpretations of regularized regression models are almost as old as the models themselves. \cite{Tibshirani1996}, the paper introducing the statistics world to LASSO regularization, points out that the LASSO estimator is the maximum \textit{a posteriori} (MAP) estimator of the regression coefficients under i.i.d.\ Laplace priors, if the residuals are Gaussian with known precision $\tau$. The nuclear-norm regularization model has a similar interpretation. Reparameterizing slightly, the estimator of \cite{ChenYuxin2020} is
\[\arg \min_{\bm{L,S}} \frac{\tau}{2} \|\bm{(L + S - Y) \circ \Omega}\|_{\text{F}}^2
+ \frac{1}{\lambda} \|\bm{L}\|_* + \frac{1}{\lambda'} \|\bm{S}\|_{1,1},\]
where $\bm{S}$ is a sparse matrix of extreme outliers and $\|\cdot\|_{1,1}$ denotes the entry-wise 1-norm. This estimator is the MAP estimator under the priors $p(\bm{L}) \propto \exp(-\|\bm{L}\|_*/\lambda)$ and $S_{ij} \overset{\mathrm{i.i.d.}}{\sim} \mathrm{Laplace}(0,\lambda')$.
Their theory explaining its strong empirical performance assumes that the residuals are i.i.d., the entries are missing completely at random, the singular vectors of $\bm{L}$ are incoherent, and the smallest singular value of $\bm{L}$ is bounded in terms of the sub-Gaussian norm of the residuals. These assumptions, though necessary for the proof, are difficult to verify in practice.
 
Frequentist analyses of the nuclear-norm regularizer typically require conditions such as incoherence of $\bm L$'s singular vectors, a minimum sampling rate on $\bm\Omega$, or entries missing completely at random to control the point estimator's worst-case error. We need none of these, because our guarantee is of a different kind. It bounds the distance from the sampler's output to the posterior induced by $\mu$, for any $\bm\Omega$, and it says nothing about how close the posterior mode is to the ground truth.

\section{Our prior}
For our proposed distribution, we will impose priors on latent factors using both the nuclear norm and the Frobenius norm:
\[\mu(\bm{L}) \propto \exp\left( -\frac{1}{\lambda}\|\bm{L}\|_*
- \frac{1}{B}\|\bm{L}\|_{\mathrm{F}}^2 \right).\]
The nuclear-norm component makes sense for shrinkage purposes to induce near-low-rankness, as documented in the earlier sections of this paper. The Frobenius-norm term has the computational advantage of making $\mu$ \emph{strongly} log-concave rather than simply log-concave, as well as the theoretical advantage of providing entrywise shrinkage in addition to the rank shrinkage of the nuclear norm. A user who is solely interested in low-rankness and only sees the Frobenius term as a vehicle for computational tractability can simply set $B$ very large compared to $\lambda$. Our bounds, however, grow polynomially in $B$ (through $K(t)$, $R_1$, and the factor $\sqrt{BR_1}$ in Theorem~\ref{thm:TVguar}), so a very large $B$ weakens the guarantees below.


One obvious shortcoming of the simple Bayesian interpretation of the Frobenius norm is the unrealistic assumption that the residual precision is known. 
As alluded to in earlier sections, and like \cite{Babacan2012} and \cite{Cui2024} but unlike most of the rest of the literature, we will add a prior on the precision. 
A natural choice, and the one \cite{Babacan2012} make, is the improper Jeffreys prior, $p(\tau) \propto 1/\tau$. This is uninformative by design. The Jeffreys prior for a parameter $q$ is defined in terms of its Fisher information matrix:
\[p(q) \propto |I(q)|^{1/2}.\]
This definition is useful because of its invariance to monotone transformation. For example, the Jeffreys prior pdf on the residual standard deviation is inversely proportional to the standard deviation, giving the same distribution as if the problem had been expressed in terms of $\tau$.
This prior is improper, though, and its use would cause trouble in later sections, since it would also cause the posterior to be improper. As such, we will truncate the support of $\tau$, giving us
\[\mu_{\star}(\tau) = \frac{1}{\tau \log \kappa}
\mathbbm{1}\{\tau_{\min} \leq \tau \leq \tau_{\max}\},\]
with $\kappa := \tau_{\max} / \tau_{\min}$.
This is the prior we will use. 
Our bounds will depend polynomially on $\tau_{\max}$ and $1/\tau_{\min}$.

\subsection{Random-walk Metropolis}
The inner task is to sample from $\mu$ and from its tilted versions $\mu_\tau$ (Section~\ref{sec:posterior}). $\mu$ is $2/B$-strongly log-concave but not smooth: its potential is the sum of a scaling of the squared Frobenius norm, which is smooth and strongly convex, and the nuclear norm, which is convex and nonsmooth but with the closed-form proximal operator $\mathrm{prox}_{c\|\cdot\|_{*}}(\bm{L}) = \bm{U}[\bm{\Sigma} - c\bm{I}]^+\bm{V}^\top$, where $\bm{L} = \bm{U}\bm{\Sigma}\bm{V}^\top$ is the SVD. Much recent work addresses this class (\cite{durmus22, Richtarik19, Richtarik20}), but Langevin Monte Carlo generally requires smoothness (\cite{Chewi2024}), and its guarantees are polynomial in the dimension and the reciprocal of the error tolerance. 
The thermodynamic-integration error bound controls each chain's bias through its $\chi^2$-divergence to the target (Lemma~\ref{lem:TI}), so each of polynomially many chains must reach a polynomially small $\chi^2$-divergence. A guarantee polynomial in the reciprocal of the tolerance would still give polynomial total complexity, but with an extra polynomial factor. A guarantee polynomial in its \emph{logarithm} avoids that factor, and the exponentially large initial divergence $u_0$ then enters only through $\log\log u_0$, so that is what we seek. The proximal Langevin sampler of \cite{Richtarik19} does not meet this criterion.
High-accuracy samplers that do include the Metropolis-adjusted Langevin algorithm (\cite{Dwivedi19}) and Metropolized Hamiltonian Monte Carlo (\cite{Lee20, Dwivedi20}). We will use random-walk Metropolis, whose analysis in \cite{Andrieu24} extends to our setting (\cite{LiuTong2026}; see Appendix \ref{app:A}). This step uses existing results; our contributions are in Section~\ref{sec:discretize} and Section~\ref{sec:posterior}.
\begin{algorithm}
\caption{Random-Walk Metropolis, RWM$(V, \pi^0, \sigma, K)$}
\label{alg:rwm}
\begin{algorithmic}[1]
\Require potential $V$ of the target; initial distribution
  $\pi^0$; step size $\sigma > 0$; number of iterations $K$
\State draw $\bm{L}^0 \sim \pi^0$
\For{$k = 0, 1, \dots, K-1$}
    \State $\bm{Z}^k \sim \mathcal{N}(\bm{L}^k, \sigma^2 \bm{I})$
    \Comment{proposal}
    \State $U^k \sim \mathrm{Unif}(0,1)$
    \If{$\log U^k \leq V(\bm{L}^k) - V(\bm{Z}^k)$}
    \Comment{symmetric proposal $\Rightarrow$ no Hastings correction}
        \State $\bm{L}^{k+1} \gets \bm{Z}^k$
    \Else
        \State $\bm{L}^{k+1} \gets \bm{L}^k$
    \EndIf
\EndFor
\State \Return $\bm{L}^K$
\end{algorithmic}
\end{algorithm}
\begin{remark}[A practical substitute for random-walk Metropolis]\label{rem:mala}
Algorithm~\ref{alg:rwm} enters Theorems~\ref{thm:TVguar} and~\ref{thm:W1guar} only through the requirement that each inner chain satisfy $\chi^2(\mu_\tau^{K(\tau)}\|\mu_\tau) \leq \varepsilon_\star^2$ (Section~\ref{sec:posterior}). Any $\mu_\tau$-invariant sampler that meets this requirement can replace it without changing the theorems, as only the step counts $K(\tau)$ in Table~\ref{tab:values} and Corollaries~\ref{cor:TV} and~\ref{cor:W1} would change.
In practice, based on our empirics, we might recommend the Metropolis-adjusted Langevin algorithm with the minimal-norm subgradient of the nuclear norm (Algorithm~\ref{alg:mala} in Appendix~\ref{app:mala}). Like random-walk Metropolis, it needs one singular value decomposition per step, and it mixed much more quickly in our experiments.
Its known non-asymptotic guarantee (\cite{LiuTong2026}) is, however, no better than the one we prove for random-walk Metropolis, so we state our results for the latter. Appendix~\ref{app:mala} gives the comparison.
\end{remark}
\subsection{Discretizing the precision}\label{sec:discretize}
Discretizing the precision is the first step of our approach. 
We define $\mu_{\star}^Q$ to be the discretization of our prior on $\tau$, consisting of $Q$ logarithmically spaced points $\tau_1, \dots, \tau_Q$ so that the prior is approximated by the uniform categorical distribution over this set.
The $q$th discretization point is therefore the $(2q-1)/2Q$ quantile of the prior:
\[\tau_q := \tau_{\min} \kappa^{(2q-1)/2Q}.\]
This approach replaces one continuous, non-log-concave problem by $Q$ strongly log-concave ones, one for each grid point.

\section{Our posterior}\label{sec:posterior}
Because residuals are assumed to be Gaussian, our likelihood is
\[p(\bm{Y}|\bm{\Omega, L,}\tau) \propto \tau^{N/2}
\exp\left(-\frac{\tau}{2} R(\bm{L})\right),\]
\[R(\bm{M}) := \|\bm{(M - Y) \circ \Omega}\|_\mathrm{F}^2.\]
Defining the tilted measure
\[\mathrm{d}\rho_s \propto \exp\left(-\frac{s}{2} R(\bm{L}) \right) \mathrm{d}\rho(\bm{L})\]
for a given measure $\rho$, we can see that our prior is $\mu = \mu_0$
and that if $\tau$ were known, the posterior would be $\mu_{\tau}$.
However, $\tau$ is not known. How do we sample from the posterior on $\bm{L}$ when $\tau$ is also uncertain? A common approach to sampling from distributions with different groups of variables is Gibbs sampling, as in \cite{Abrahamsen2017, Hobert1998, Jones2004, Roman2012, Rosenthal1995}. Explicit convergence bounds exist for Gibbs samplers in some hierarchical variance-component models (\cite{Rosenthal1995, Jones2004}), but we are not aware of any for this model. Moreover, the joint posterior over $(\bm{L},\tau)$ is non-log-concave (Lemma~\ref{lem:nonlc}) and non-smooth, so the usual non-asymptotic guarantees do not apply to it directly. One could instead integrate $\tau$ out analytically: the marginal posterior on $\bm{L}$ has potential 
\[\frac{1}{\lambda}\|\bm{L}\|_* + \frac{1}{B}\|\bm{L}\|_{\mathrm{F}}^2 + g(R(\bm{L})),\] 
where 
\[g(r) := -\log\int_{\tau_{\min}}^{\tau_{\max}} \tau^{N/2-1} e^{-\tau r/2}\,\mathrm{d}\tau\] 
is concave and increasing, with slope in $[\tau_{\min}/2, \tau_{\max}/2]$. This potential is not convex in general. 

\begin{algorithm}[!t]
\caption{Thermodynamic-integration marginal-likelihood estimator}
\label{alg:ti}
\begin{algorithmic}[1]
\Require step-size function $\sigma(\cdot)$; iteration function $K(\cdot)$; grid $\tau_1<\cdots<\tau_Q$; inner sample size $M$
\State $\widetilde\Lambda_1 \gets 0$
\For{$i \in \{1,\ldots,Q\}$}
  \For{$j \in \{1,\ldots,M\}$}
    \State $\bm{L}_j^{(i)} \gets \mathrm{RWM}(V_{\tau_i}, \mu_{\tau_i}^0, \sigma(\tau_i), K(\tau_i))$ 
  \EndFor
  \State $\widehat R_{M}(\tau_i) \gets \dfrac{1}{M}\sum_{j=1}^{M} R(\bm{L}_j^{(i)})$
  \State $\widehat f(\tau_i) \gets \dfrac{N}{2\tau_i} - \dfrac12 \widehat R_{M}(\tau_i)$
\EndFor
\For{$q \in \{2,\ldots,Q\}$}
  \State $\widetilde\Lambda_q \gets \displaystyle\sum_{i=1}^{q-1}\frac{\widehat f(\tau_i) + \widehat f(\tau_{i+1})}{2}\,(\tau_{i+1}-\tau_i)$
\EndFor
\State \Return $\nu_\star^{Q,\mathrm{TI}} \gets \mathrm{softmax}(\{\widetilde\Lambda_1, \dots, \widetilde\Lambda_Q\})$
\end{algorithmic}
\end{algorithm}
\begin{algorithm}[!t]
\caption{Posterior sampler}
\label{alg:post-samp}
\begin{algorithmic}[1]
\Require distribution $\nu_\star^{Q,\mathrm{TI}}$
\For{$i \in \{1, \dots, \mathrm{num\_samples}\}$}
    \State $\tau^{(i)} \sim \nu_\star^{Q,\mathrm{TI}}$
    \State $\bm{L}'_i \gets \mathrm{RWM}(V_{\tau^{(i)}}, \mu_{\tau^{(i)}}^0, \sigma(\tau^{(i)}), K(\tau^{(i)}))$
\EndFor
\State \Return $\{\bm{L}'_i : i \in [\mathrm{num\_samples}]\}$
\end{algorithmic}
\end{algorithm}

Instead, we sample from the \textit{marginal} posterior of the precision and then from the conditional posterior of $\bm{L}$. This is the main construction of the paper: the marginal posterior of the precision is estimated on the grid by thermodynamic integration, which samples only from the strongly log-concave tilted measures $\mu_{\tau_q}$. We use the following notation.
\begin{itemize}
    \item $\nu_{\star}$ is the posterior measure on $\tau$.
    \item $\bm{\Lambda}$ is the vector of marginal log-likelihoods over the $\tau_q$'s; that is,
    \[\Lambda_q := \log \mathbb{E}_{\bm{L} \sim \mu}[p(\bm{Y}|\bm{L},\tau_q)].\]
    \item $\nu_{\star}^Q := \mathrm{softmax}(\bm{\Lambda})$
    is the categorical distribution induced by normalizing $\bm{\Lambda}$.
    \item $\widetilde{\Lambda}$ is the vector of approximate marginal
    log-likelihoods over the $\tau_q$'s, up to an additive constant.
    $\widetilde{\Lambda}_q$ estimates $\Lambda_q - \Lambda_1$ by thermodynamic integration of the score $\frac{d}{d\tau}\Lambda = \frac{N}{2\tau} - \frac12 \mathbb{E}_{\mu_\tau}[R(L)]$ along the grid (Algorithm~\ref{alg:ti}). The constant does not affect $\mathrm{softmax}(\widetilde{\Lambda})$.
    \item $\nu_{\star}^{Q,\mathrm{TI}} := \mathrm{softmax}(\widetilde{\bm{\Lambda}})$
    is the categorical distribution induced by normalizing $\widetilde{\bm{\Lambda}}$. It is random, because $\widetilde{\bm{\Lambda}}$ is estimated from Monte Carlo draws, and $\mathbb{E}[\nu_{\star}^{Q,\mathrm{TI}}]$ denotes its expectation over the randomness of Algorithm~\ref{alg:ti}.
    \item $\mu_{\tau}$ is the posterior measure on $\bm{L}$ conditional on $\tau$, from which we cannot sample. Its potential is $V_{\tau}$.
    \item $\mu^0_\tau$ is the initial distribution of every chain at precision $\tau$:
    \[\mu^0_\tau(\bm{L}') \propto \exp\left(-\frac1B\|\bm{L}'\|_{\mathrm{F}}^2
      - \frac{\tau}{2}\|(\bm{L}'-\bm{Y}) \circ \bm{\Omega}\|_{\mathrm{F}}^2\right),\]
    the conditional posterior with the nuclear-norm term dropped. Unobserved entries have distribution $\mathcal{N}(0, B/2)$ and observed ones are $\mathcal{N}\big(\tfrac{\tau Y_{ij}}{2/B+\tau}, \tfrac{1}{2/B+\tau}\big)$, so it can be sampled exactly.
    \item $\mu_{\tau}^{K(\tau)}$ is the law of the output of Algorithm \ref{alg:rwm} after $K(\tau)$ iterations starting at $\mu_{\tau}^0$, where $K(\tau)$ is chosen so that $\chi^2(\mu_{\tau}^{K(\tau)}\|\mu_{\tau}) \leq \varepsilon_{\star}^2$.
\end{itemize}
Our goal is now to sample from $\int \mu_{\tau} \nu_{\star}$, which denotes the marginal posterior of $\bm{L}$ (the conditional posterior integrated against the marginal posterior of $\tau$). 
Algorithm~\ref{alg:post-samp} then draws $\tau$ from $\nu_{\star}^{Q,\text{TI}}$ and $\bm{L}$ from $\mu_{\tau}^{K(\tau)}$.

Because the smooth part of $\mu_t$'s potential is
$2/B$-strongly convex and $(2/B + t)$-smooth
and the nuclear-norm term is $(\sqrt{n_1}/\lambda)$-Lipschitz,
and $\mathbb{E}_{\mu_t^0}[\bm{L}] = (1 - \frac{2}{Bt + 2}) \bm{Y\circ\Omega}$,
the error of the algorithm is as follows.
Note that the notation is in terms of mixtures of probability distributions,
with a conditional law of $\bm{L}$ given $\tau$ and a marginal law of $\tau$ written next to each other in an integrand.

\begin{table}[h]
\centering
\small
\caption{Glossary of auxiliary quantities}
\label{tab:glossary}
\begin{tabular}{lp{6cm}r}
\toprule
Symbol & Role \\
\midrule
$\sigma(t)$ & RWM proposal step size at $\tau=t$ \\
$A(t)$ & Intermediate bound on $\mathbb E\|\boldsymbol L\|_{\mathrm{F}}^2$ feeding $u_0$ \\
$u_0(t)$ & Burn-in chi-squared divergence $\chi^2(\mu_t^0\|\mu_t)$ \\
$R_1$ & Bound on $\mathbb E_\mu[R(\boldsymbol L)]$ \\
$R_2$ & Bound on $\mathbb E_\mu[R(\boldsymbol L)^2]$ \\
$\widetilde R_2$ & $R_1^2+2BR_1$ \\
$\widehat R_4$ & Bound on the centered 4th moment of $R(\boldsymbol L)$ \\
$L_f$ & Lipschitz constant of the TI score function \\
$\widehat V$ & Variance bound on $R(\boldsymbol L)$ under the finite-$K$ chain \\
$\Delta$ & Maximum $\tau$-grid interval width \\
$K(t)$ & RWM steps per chain at $\tau=t$ \\
$\mathcal W_1$ & Bound on $\mathbb E[W_1(\nu_\star,\nu_\star^{Q,\mathrm{TI}})]$ at this $(Q,\varepsilon_\star)$ \\
$\varepsilon_\star$ & Target accuracy of each inner chain:
  $\chi^2(\mu_\tau^{K(\tau)}\|\mu_\tau)\le\varepsilon_\star^2$ \\
$C_\ell$ & Isoperimetric constant, $0.958357$ \\
$\bar u_0(t)$ & $\max\{u_0(t), 8\}$ \\
\bottomrule
\end{tabular}
\end{table}

\begin{theorem}\label{thm:TVguar}
If the values of the quantities in Table \ref{tab:glossary} are as in Table \ref{tab:values} (see Appendix \ref{sec:proofs}), the law of a single output of Algorithm~\ref{alg:post-samp} satisfies
\begin{align*}
&\mathrm{TV}\left(\int \mu_{\tau} \nu_{\star} (\mathrm{d}\tau),
\int \mu_{\tau}^{K(\tau)} \mathbb{E}[\nu^{Q,\mathrm{TI}}_{\star}] (\mathrm{d}\tau) \right) &&\\
&\leq \frac{\varepsilon_{\star}}{2} + \sqrt{\frac{BR_1}{8}}\,\mathcal{W}_1.
\end{align*}
Moreover, for any $p \in (0,1)$, with probability at least $1-p$ over the randomness of Algorithm~\ref{alg:ti},
\begin{align*}
&\mathrm{TV}\left(\int \mu_{\tau} \nu_{\star} (\mathrm{d}\tau),
\int \mu_{\tau}^{K(\tau)} \nu^{Q,\mathrm{TI}}_{\star} (\mathrm{d}\tau) \right) &&\\
&\leq \frac{\varepsilon_{\star}}{2} + \sqrt{\frac{BR_1}{8}}\,\frac{\mathcal{W}_1}{p}.
\end{align*}
\end{theorem}

\begin{theorem}\label{thm:W1guar}
Under the same setting as Theorem \ref{thm:TVguar}, the law of a single output of Algorithm~\ref{alg:post-samp} satisfies
\begin{align*}
&W_1\left(\int \mu_{\tau} \nu_{\star} (\mathrm{d}\tau),
\int \mu_{\tau}^{K(\tau)} \mathbb{E}[\nu^{Q,\mathrm{TI}}_{\star}] (\mathrm{d}\tau) \right) &&\\
&\leq \sqrt{B}\,\varepsilon_{\star} + \frac{B\sqrt{R_1}}{2}\,\mathcal{W}_1,
\end{align*}
where $W_1$ is taken with respect to the Frobenius norm. Moreover, for any $p \in (0,1)$, with probability at least $1-p$ over the randomness of Algorithm~\ref{alg:ti},
\begin{align*}
&W_1\left(\int \mu_{\tau} \nu_{\star} (\mathrm{d}\tau),
\int \mu_{\tau}^{K(\tau)} \nu^{Q,\mathrm{TI}}_{\star} (\mathrm{d}\tau) \right) &&\\
&\leq \sqrt{B}\,\varepsilon_{\star} + \frac{B\sqrt{R_1}}{2}\,\frac{\mathcal{W}_1}{p}.
\end{align*}
\end{theorem}

The proofs (Appendix \ref{sec:proofs}) use the triangle inequality to decompose the error in terms of error from each source.
The error in the precision posterior is bounded in $W_1$ by the discretization error (Lemma~\ref{lem:Q}) plus the thermodynamic-integration error (Lemma~\ref{lem:TI}). 
By Pinsker's inequality and Lemma~\ref{lem:bregman}, it then moves the mixture over conditional posteriors of $\bm{L}$ by an amount linear in that $W_1$ error, which is the term $\mathcal{W}_1$ in both theorems.
These proofs show where our new ingredients enter.

The first bound in each theorem concerns the law of a single output of Algorithm~\ref{alg:post-samp}, averaged over the randomness of Algorithm~\ref{alg:ti}. Since all outputs of Algorithm~\ref{alg:post-samp} share one realization of $\nu_\star^{Q,\mathrm{TI}}$ and are i.i.d.\ given it, the second bound is the relevant one for posterior summaries computed from many draws.

\begin{figure*}[t]
    \centering
    \includegraphics[width=0.82\textwidth]{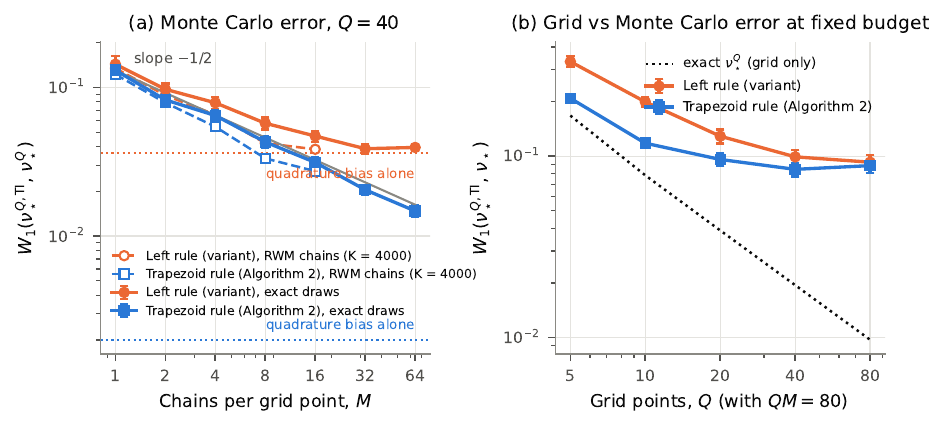}
    \caption{Exact benchmark ($\lambda = \infty$, $N = 109$). (a) $W_1$ distance between the thermodynamic-integration posterior $\nu_\star^{Q,\mathrm{TI}}$ and the exact discretized posterior $\nu_\star^Q$ as the number of chains per grid point grows, with exact draws (solid) and random-walk chains (dashed); dotted lines mark each rule's quadrature error alone. (b) $W_1$ distance to the exact continuous posterior $\nu_\star$ at a fixed budget $QM = 80$; the dotted line is the error of the grid alone. Error bars are 95\% confidence intervals over replicates.}
    \label{fig:exact}
\end{figure*}
\begin{corollary}\label{cor:TV}
If our prior satisfies 
\[\tau_{\max} = \Theta(1), \quad \tau_{\min} = \Theta(1), \quad B = \Theta(1),\]
\[\lambda = \Theta(n_1^a), \quad a \in \mathbb{R},\]
and we standardize the observations $\bm{Y}$ such that
\[\sum_{i,j} Y_{ij} \Omega_{ij} = 0, \quad 
\sum_{i,j} Y_{ij}^2 \Omega_{ij} = \|\bm{Y\circ\Omega}\|_{\mathrm F}^2 = \Theta(N),\]
we can achieve $\eta$ TV distance 
in the first bound of Theorem~\ref{thm:TVguar} 
in
\[QMK = \widetilde{O}\!\left( \frac{N^2 n_1 n_2 \max\{n_1^{1-2a}, 1\}}{\eta^2} \right)\]
RWM steps.
\end{corollary}

\begin{corollary}\label{cor:W1}
Under the same setting as Corollary \ref{cor:TV},
we can achieve $\eta$ W1 distance with respect to the Frobenius norm 
in the first bound of Theorem~\ref{thm:W1guar} 
in
\[QMK = \widetilde{O}\!\left( \frac{N^2 n_1 n_2 \max\{n_1^{1-2a}, 1\}}{\eta^2} \right)\]
RWM steps. 
In particular, the W1 distance between the posterior marginal of each entry of $\bm{L}$ and its sampled counterpart is then at most $\eta$.
\end{corollary}

The $\eta^{-2}$ comes from the Monte Carlo term $N/\sqrt{QM}$, the usual Monte Carlo rate. The grid terms alone need only $Q \asymp N^{3/2}/\eta$, and the product $QM$ is what matters. Improving the dependence on $\eta$ would require variance reduction rather than a finer analysis of the grid. Replacing $\eta$ by $p\eta/2$ in Corollary \ref{cor:TV} bounds the second distance in Theorem \ref{thm:TVguar} by $\eta$ with probability at least $1-p$.
Standardization of observations is a common practice in empirical studies, 
and to match the scale given by the prior, we scale the entries as in Corollary~\ref{cor:TV}. 


\section{Numerical experiments}
\label{sec:numerical}
\begin{figure*}[t]
    \centering
    \includegraphics[width=0.82\textwidth]{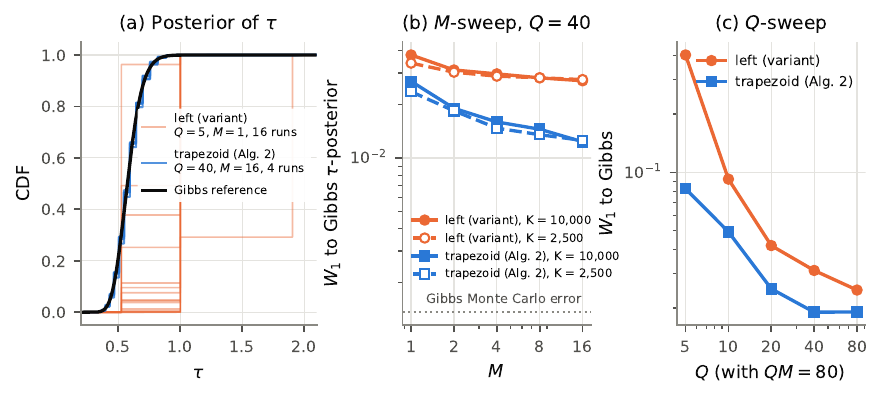}
    \caption{Posterior of $\tau$ at $\lambda = 0.2$ against a long Gibbs reference. (a) CDFs from independent runs on a coarse grid with the left-Riemann rule ($Q = 5$, $M = 1$) and on a fine grid with the trapezoid rule ($Q = 40$, $M = 16$). (b) $W_1$ distance to the reference as $M$ grows, at two chain lengths $K$. (c) $W_1$ distance at a fixed budget $QM = 80$. The dotted line in (b) is the Monte Carlo error of the reference.}
    \label{fig:lambda02}
\end{figure*}

Our results are worst-case guarantees, and the step counts they prescribe (Table~\ref{tab:values}) are far beyond what can be run. The experiments below ask three narrower questions: whether Algorithms~\ref{alg:ti} and~\ref{alg:post-samp} recover the posterior at practical budgets, how loose the bounds are, and how thermodynamic integration compares with other estimators of the marginal likelihood. Implementation details, further results and a real-data example can be found in Appendix~\ref{app:experiments}. Code and data can be found in the supplementary material.

\paragraph{Setup.} Unless stated otherwise, the data are a $12 \times 12$ rank-2 matrix with singular values $(12, 8)$ and noise precision $\tau = 1$, with each entry observed with probability 0.7 (yielding $N = 109$ in our experiments), $B = 2$ and $[\tau_{\min}, \tau_{\max}] = [0.2, 5]$. The step size of Algorithm~\ref{alg:rwm} is tuned in a discarded pilot run to about $25\%$ acceptance, and the step size of Table~\ref{tab:values} is 15 to 80 times smaller and accepts 94--99\% of proposals. 
The experiments therefore depart from the analyzed algorithm in three ways: they use tuned step sizes, chains of 2,500--40,000 steps rather than the $3\times10^{10}$--$10^{12}$ prescribed, and, on the panel data and in the coverage study of Appendix~\ref{app:experiments}, the Metropolis-adjusted Langevin algorithm (Algorithm~\ref{alg:mala}) rather than random-walk Metropolis.
They test the algorithms and the error decomposition of Lemma~\ref{lem:TI}, not the step counts of the theorems.

\paragraph{Validation against an exact benchmark.} With $\lambda = \infty$, every $\mu_\tau$ is Gaussian, so $\Lambda_\mu$, $\nu_\star^Q$, $\nu_\star$ and every posterior marginal of $\bm{L}$ are available in closed form. The error terms of Lemma~\ref{lem:TI} can be measured separately in this case (Figure~\ref{fig:exact}). The Monte Carlo error decreases as $M^{-1/2}$ (fitted log-log slope $-0.51$), as Lemma~\ref{lem:TI} predicts, until it reaches the quadrature error. That floor depends on the rule. With exact expectations, the error of the left-Riemann rule in $\Lambda_\mu(\tau_q) - \Lambda_\mu(\tau_1)$ is 9.0, 2.2 and 0.56 nats at $Q = 10$, 40 and 160, against 1.3, 0.085 and 0.005 for the trapezoid rule. 
With $Q = 40$ and $M = 16$, the median over entries of the $W_1$ distance between the output of Algorithm~\ref{alg:post-samp} and the exact posterior marginal is 0.013--0.015, against 0.012 for exact posterior draws of the same size. The bound $\mathcal{W}_1$ of Table~\ref{tab:values} exceeds the measured $\mathbb{E}[W_1(\nu_\star^{Q,\mathrm{TI}}, \nu_\star)]$ by factors of $3 \times 10^4$ to $5 \times 10^5$.

\paragraph{The precision posterior at $\lambda = 0.2$.} With the nuclear norm, no closed form is available, so we compare with a long Metropolis-within-Gibbs reference (the ``Gibbs reference'' below): 48 chains of 6{,}000 sweeps (split-$\widehat{R} = 1.003$), each sweep drawing $\tau \mid \bm{L}$ exactly from a truncated gamma distribution and updating $\bm{L} \mid \tau$ with random-walk Metropolis steps. The reference posterior of $\tau$ has mean 0.58 and standard deviation 0.10: below the true $\tau = 1$, because the prior's shrinkage leaves larger residuals. On a coarse grid ($Q = 5$), the thermodynamic-integration posterior concentrates on one or two grid points (Figure~\ref{fig:lambda02}a): the left-Riemann rule puts 88\% of its mass on $\tau = 1.0$, at $W_1$ distance 0.39 from the reference. Raising $M$ from 1 to 16 does not change this, so the concentration comes from the coarse grid and the quadrature bias rather than from Monte Carlo noise; the trapezoid rule on the same grid gives $W_1 = 0.12$. With $Q = 40$, the distance falls to 0.012--0.04 (Figure~\ref{fig:lambda02}b,c), and runs with $K = 2{,}500$ and $K = 10{,}000$ agree. 
With $Q = 40$, $M = 16$ and the trapezoid rule, the per-entry $W_1$ distance between the output of Algorithm~\ref{alg:post-samp} and the Gibbs draws is close to the Gibbs sampler's own Monte Carlo error, estimated by comparing chains that share no draws (medians 0.024 and 0.024 on observed entries, 0.029 and 0.026 on missing ones).

\paragraph{Cost.} Each step of Algorithm~\ref{alg:rwm}, Algorithm~\ref{alg:mala}, and the Gibbs reference costs one SVD. 
At $\lambda = 0.2$, Algorithm~\ref{alg:ti} with $Q = 40$, $M = 16$ and $K = 10{,}000$ takes $6.4\times10^6$ random-walk steps, and each output of Algorithm~\ref{alg:post-samp} takes $10^4$ more, so the 1,000 draws above used $1.6\times10^7$ steps in all, or $1.6\times10^4$ per draw, and these draws are independent given $\nu_\star^{Q,\mathrm{TI}}$. 
The Gibbs reference used $5.8\times10^6$ steps (48 chains of 6,000 sweeps of 20 steps, including 1{,}000 sweeps of burn-in), from which we retain every 50th sweep for 4,800 draws in all. 
Their bulk effective sample size is 3,970 for $R(\bm L)$, 4,180 for $\|\bm L\|_*$ and 4,620 for $\tau$, or about 1,250--1,450 steps per effective draw. On this problem, Algorithms~\ref{alg:ti} and\ref{alg:post-samp} at this budget therefore cost 11--13 times more computation per effective draw than a well-mixing Gibbs sampler.
The cost of Algorithm~\ref{alg:ti} is paid once, and the outputs of Algorithm~\ref{alg:post-samp} run in parallel.

\paragraph{Comparison with other estimators.} At an equal budget of 640 exact draws on the $\lambda = \infty$ benchmark with $Q = 40$, the mean over replicates of $\max_q |\widetilde\Lambda_q - (\Lambda_\mu(\tau_q) - \Lambda_\mu(\tau_1))|$, in nats, is as follows (TI: thermodynamic integration).
\begin{center}
\small
\begin{tabular}{lrrr}
\toprule
Estimator & $N = 25$ & $N = 109$ & $N = 745$ \\
\midrule
Importance sampling & 65 & 396 & 3830 \\
TI, left rule & 0.67 & 2.5 & 15.1 \\
Stepping-stone & 0.47 & 1.0 & 5.7 \\
TI, trapezoid rule & 0.45 & 0.86 & 2.4 \\
Bridge (geometric) & 0.45 & 0.88 & 2.5 \\
\bottomrule
\end{tabular}
\end{center}
Importance sampling from the untilted prior fails by orders of magnitude, as Lemma~\ref{lem:is} predicts. Among the estimators that sample the tilted measures, the trapezoid rule and bridge sampling are best, and the left-Riemann rule is worst.

\paragraph{Theory and practice.} On the toy problem, Table~\ref{tab:values} prescribes $K$ between $3 \times 10^{10}$ ($\lambda = 1$) and $10^{12}$ ($\lambda = 0.2$) steps per chain, whereas tuned random-walk chains reach split-$\widehat{R} < 1.05$ on $R(\bm{L})$ and $\|\bm{L}\|_*$ within 14{,}700--27{,}500 steps. 
A rough scaling study on the exact benchmark (Appendix~\ref{app:experiments}) suggests that the total number of random-walk steps needed to bring $W_1(\nu_\star^{Q,\mathrm{TI}}, \nu_\star)$ below a quarter of the posterior standard deviation grows like $N^{1.74}$ for $N$ from 25 to 384. Its target differs from the fixed accuracy $\eta$ of Corollary~\ref{cor:TV}, so it shows only that practical costs grow much more slowly than the bound, not a rate to compare with $N^3$. 
Random-walk Metropolis does, however, become slow on larger problems: on the $31 \times 39$ panel of Appendix~\ref{app:experiments}, 20{,}000 steps leave split-$\widehat{R}$ at 5.0, whereas the Metropolis-adjusted Langevin algorithm of Remark~\ref{rem:mala} reaches 1.01.

\paragraph{Limitations: the posterior itself.} Our guarantees concern sampling, not the statistical quality of the posterior, which depends strongly on $\lambda$ when $\tau$ is unknown. 
On a recoverable $12 \times 12$ instance with true $\tau = 10$ (Appendix~\ref{app:experiments}), taking $\lambda = 0.05$, which puts the singular-value threshold $1/(\lambda\tau)$ of the posterior mode at the noise level when $\tau$ is the true value, concentrates the posterior of $\tau$ near $\tau_{\min} = 1$: the posterior attributes most of the signal to noise, its 95\% credible intervals cover 37--39\% of held-out entries, and its mean is no more accurate than predicting zero. The penalty itself is not at fault: the posterior mode at the true $\tau$ reduces the held-out error by 35\% at 70\% observed. The mechanism is that $\lambda$ is fixed while $\tau$ is learned: the prior penalizes $\|\bm{L}\|_*$ on an absolute scale, so a smaller $\tau$ lets the posterior explain the data as noise with a small $\bm{L}$ at little cost. \cite{ParkCasella2008} condition the Bayesian lasso prior on the noise variance to rule out this kind of behavior. With $\lambda = 0.5$, coverage is 88--91\%, but the posterior mean of $\tau$ ranges from 1.4 to 56 across data sets. Split-$\widehat{R}$ was at most 1.03 in every run, so this is a property of the posterior rather than of the sampler, consistent with the suboptimal contraction of the Bayesian lasso posterior (\cite{Castillo2015}).
Choosing $\lambda$, $B$ and $[\tau_{\min}, \tau_{\max}]$ optimally is outside the scope of this paper. 

\bibliography{refs}

@article{Abadie2010,
    author = {Alberto Abadie and Alexis Diamond and Jens Hainmueller},
    title = {Synthetic Control Methods for Comparative Case Studies: Estimating the Effect of {C}alifornia’s Tobacco Control Program},
    journal = {Journal of the American Statistical Association},
    volume = {105},
    number = {490},
    pages = {493--505},
    year = {2010},
    publisher = {ASA Website},
    doi = {10.1198/jasa.2009.ap08746},
    URL = {https://doi.org/10.1198/jasa.2009.ap08746}
}

@article{Abrahamsen2017,
    author = {Tavis Abrahamsen and James P. Hobert},
    title = {{Convergence analysis of block Gibbs samplers for Bayesian linear mixed models with $p>N$}},
    volume = {23},
    journal = {Bernoulli},
    number = {1},
    publisher = {Bernoulli Society for Mathematical Statistics and Probability},
    pages = {459 -- 478},
    year = {2017},
    doi = {10.3150/15-BEJ749},
    URL = {https://doi.org/10.3150/15-BEJ749}
}

@inproceedings{agarwal21causal,
  author    = {Agarwal, Anish and Dahleh, Munther and Shah, Devavrat and Shen, Dennis},
  title     = {Causal Matrix Completion},
  booktitle = {Proceedings of the Thirty Sixth Conference on Learning Theory},
  series    = {Proceedings of Machine Learning Research},
  volume    = {195},
  pages     = {3821--3826},
  year      = {2023},
  editor    = {Neu, Gergely and Rosasco, Lorenzo},
  publisher = {PMLR}
}

@article{Andrieu24,
  author  = {Andrieu, Christophe and Lee, Anthony and Power, Sam and Wang, Andi Q.},
  title   = {Explicit convergence bounds for {M}etropolis {M}arkov chains: isoperimetry, spectral gaps and profiles},
  journal = {Annals of Applied Probability},
  volume  = {34},
  number  = {4},
  pages   = {4022--4071},
  year    = {2024},
  doi     = {10.1214/24-AAP2058}
}

@article{athey2021matrix,
  title={Matrix completion methods for causal panel data models},
  author={Athey, Susan and Bayati, Mohsen and Doudchenko, Nikolay and Imbens, Guido and Khosravi, Khashayar},
  journal={Journal of the American Statistical Association},
  volume={116},
  number={536},
  pages={1716--1730},
  year={2021},
  publisher={Taylor \& Francis}
}

@article{Babacan2012,
  author  = {Babacan, S. Derin and Luessi, Martin and Molina, Rafael and Katsaggelos, Aggelos K.},
  title   = {Sparse {B}ayesian methods for low-rank matrix estimation},
  journal = {IEEE Transactions on Signal Processing},
  volume  = {60},
  number  = {8},
  pages   = {3964--3977},
  year    = {2012}
}

@article{bai03,
author = {Bai, Jushan},
title = {Inferential Theory for Factor Models of Large Dimensions},
year = {2003},
issue_date = {January 2003},
publisher = {Wiley},
address = {Hoboken, NJ, USA},
volume = {71},
number = {1},
pages = {135-171},
url = {https://doi.org/10.1111/1468-0262.00392},
doi = {10.1111/1468-0262.00392},
journal = {Econometrica},
month = {5},
numpages = {37}
}

@article{bai21,
author = {Bai, Jushan and Ng, Serena},
title = {Matrix Completion, Counterfactuals, and Factor Analysis of Missing Data},
year = {2021},
issue_date = {20 September 2021},
publisher = {Taylor \& Francis},
address = {London, UK},
volume = {116},
number = {536},
pages = {1746-1763},
url = {https://www.tandfonline.com/doi/full/10.1080/01621459.2021.1967163},
doi = {10.1080/01621459.2021.1967163},
journal = {Journal of the American Statistical Association},
numpages = {18}
}

@article{Banerjee2005,
  author  = {Banerjee, Arindam and Merugu, Srujana and Dhillon, Inderjit S. and Ghosh, Joydeep},
  title   = {Clustering with {B}regman divergences},
  journal = {Journal of Machine Learning Research},
  volume  = {6},
  pages   = {1705--1749},
  year    = {2005}
}

@article{Baxendale2005,
  author  = {Baxendale, Peter H.},
  title   = {Renewal theory and computable convergence rates for geometrically ergodic {M}arkov chains},
  journal = {The Annals of Applied Probability},
  volume  = {15},
  number  = {1B},
  pages   = {700--738},
  year    = {2005}
}

@article{Brosse2018,
  author  = {Brosse, Nicolas and Durmus, Alain and Moulines, {\'E}ric},
  title   = {Normalizing constants of log-concave densities},
  journal = {Electronic Journal of Statistics},
  volume  = {12},
  number  = {1},
  pages   = {851--889},
  year    = {2018}
}

@article{Caffarelli2000,
  author  = {Caffarelli, Luis A.},
  title   = {Monotonicity properties of optimal transportation and the {FKG} and related inequalities},
  journal = {Communications in Mathematical Physics},
  volume  = {214},
  number  = {3},
  pages   = {547--563},
  year    = {2000},
  doi     = {10.1007/s002200000257},
  note    = {Erratum: \emph{Comm.\ Math.\ Phys.} 225 (2002), 449--450}
}

@article{candestao10,
author = {Candès, Emmanuel and Tao, Terence},
title = {The Power of Convex Relaxation: Near-Optimal Matrix Completion},
year = {2010},
issue_date = {5 May 2010},
publisher = {IEEE},
volume = {56},
number = {5},
pages = {2053-2080},
url = {https://ieeexplore.ieee.org/document/5452187},
doi = {10.1109/TIT.2010.2044061},
journal = {IEEE Transactions on Information Theory},
month = {5},
numpages = {28}
}

@article{Castillo2015,
  author  = {Castillo, Isma{\"e}l and Schmidt-Hieber, Johannes and van der Vaart, Aad},
  title   = {Bayesian linear regression with sparse priors},
  journal = {The Annals of Statistics},
  volume  = {43},
  number  = {5},
  pages   = {1986--2018},
  year    = {2015}
}

@article{Dwivedi20,
  author = {Chen, Yuansi and Dwivedi, Raaz and Wainwright, Martin J. and Yu, Bin},
  title = {Fast mixing of {M}etropolized {H}amiltonian {M}onte {C}arlo: benefits of multi-step gradients},
  journal = {Journal of Machine Learning Research},
  year = {2020},
  volume = {21},
  number = {92},
  pages = {1--72},
  url = {https://jmlr.org/papers/v21/19-441.html}
}

@article{ChenYuxin2020,
  author = {Yuxin Chen and Jianqing Fan and Cong Ma and Yuling Yan},
  title = {Bridging Convex and Nonconvex Optimization in Robust {PCA}: Noise, Outliers, and Missing Data},
  journal = {The Annals of Statistics},
  year = {2021},
  volume = {49},
  number = {5},
  pages = {2948-2971},
  doi = {10.1214/21-aos2066},
}

@inproceedings{ChenChewi2022,
  author    = {Chen, Yongxin and Chewi, Sinho and Salim, Adil and Wibisono, Andre},
  title     = {Improved analysis for a proximal algorithm for sampling},
  booktitle = {Proceedings of the Thirty Fifth Conference on Learning Theory},
  editor    = {Loh, Po-Ling and Raginsky, Maxim},
  series    = {Proceedings of Machine Learning Research},
  volume    = {178},
  pages     = {2984--3014},
  year      = {2022},
  publisher = {PMLR}
}

@misc{Chewi2024,
  author = {Sinho Chewi},
  title = {Log-Concave Sampling},
  year = {2026},
  note = {Book draft, version of August 9, 2026},
  url = {https://chewisinho.github.io/main.pdf}
}

@article{Cui2024,
  author = {Cui, Tiangang and Gorodetsky, Alex A.},
  title = {Low-rank {B}ayesian matrix completion via geodesic {H}amiltonian {M}onte {C}arlo on {S}tiefel manifolds},
  journal = {Foundations of Data Science},
  year = {2026},
  doi = {10.3934/fods.2026015},
  url = {https://www.aimsciences.org/article/id/6a2a7f9583e15054ce8a2242},
  pages = {},
  issn = {},
}

@article{durmus22,
author = {Durmus, Alain and Moulines, \'{E}ric and Pereyra, Marcelo},
title = {A Proximal {M}arkov Chain {M}onte {C}arlo Method for {B}ayesian Inference in Imaging Inverse Problems: When {L}angevin Meets {M}oreau},
journal = {SIAM Review},
volume = {64},
number = {4},
pages = {991-1028},
year = {2022},
doi = {10.1137/22M1522917},
URL = {https://doi.org/10.1137/22M1522917},
  note = {Reprint of \emph{SIAM Journal on Imaging Sciences}, 11(1):473--506, 2018}
}

@article{Dwivedi19,
  author  = {Raaz Dwivedi and Yuansi Chen and Martin J. Wainwright and Bin Yu},
  title   = {Log-concave sampling: {M}etropolis--{H}astings algorithms are fast},
  journal = {Journal of Machine Learning Research},
  year    = {2019},
  volume  = {20},
  number  = {183},
  pages   = {1--42},
  url     = {http://jmlr.org/papers/v20/19-306.html}
}

@inproceedings{FanYuanChen2023,
  author    = {Fan, Jiaojiao and Yuan, Bo and Chen, Yongxin},
  title     = {Improved dimension dependence of a proximal algorithm for sampling},
  booktitle = {Proceedings of the Thirty Sixth Conference on Learning Theory},
  series    = {Proceedings of Machine Learning Research},
  volume    = {195},
  pages     = {1473--1521},
  year      = {2023}
}

@inproceedings{Ge2020,
  author    = {Ge, Rong and Lee, Holden and Lu, Jianfeng},
  title     = {Estimating normalizing constants for log-concave distributions: algorithms and lower bounds},
  booktitle = {Proceedings of the 52nd Annual ACM SIGACT Symposium on Theory of Computing},
  pages     = {579--586},
  year      = {2020}
}

@article{GelmanMeng1998,
  author  = {Gelman, Andrew and Meng, Xiao-Li},
  title   = {Simulating normalizing constants: from importance sampling to bridge sampling to path sampling},
  journal = {Statistical Science},
  volume  = {13},
  number  = {2},
  pages   = {163--185},
  year    = {1998}
}

@inproceedings{Guo2026,
  author    = {Guo, Wei and Tao, Molei and Chen, Yongxin},
  title     = {Complexity analysis of normalizing constant estimation: from {J}arzynski equality to annealed importance sampling and beyond},
  booktitle = {International Conference on Learning Representations},
  year      = {2026}
}

@article{Hobert1998,
    title = {Geometric Ergodicity of {G}ibbs and Block {G}ibbs Samplers for a Hierarchical Random Effects Model},
    journal = {Journal of Multivariate Analysis},
    volume = {67},
    number = {2},
    pages = {414-430},
    year = {1998},
    issn = {0047-259X},
    doi = {https://doi.org/10.1006/jmva.1998.1778},
    url = {https://www.sciencedirect.com/science/article/pii/S0047259X9891778X},
    author = {James P. Hobert and Charles J. Geyer},
}

@article{Jones2004,
    author = {Galin L. Jones and James P. Hobert},
    title = {{Sufficient burn-in for Gibbs samplers for a hierarchical random effects model}},
    volume = {32},
    journal = {The Annals of Statistics},
    number = {2},
    publisher = {Institute of Mathematical Statistics},
    pages = {784 -- 817},
    year = {2004},
    doi = {10.1214/009053604000000184},
    URL = {https://doi.org/10.1214/009053604000000184}
}

@article{Keshavan2010Main,
  author = {Raghunandan H. Keshavan and Andrea Montanari and Sewoong Oh},
  title = {Matrix Completion from Noisy Entries},
  journal = {Journal of Machine Learning Research},
  year = {2010},
  volume = {11},
  pages = {2057-2078},
}

@article{Keshavan2010Later,
  author = {Raghunandan H. Keshavan and A. Montanari},
  title = {Regularization for matrix completion},
  journal = {2010 IEEE International Symposium on Information Theory},
  year = {2010},
  pages = {1503-1507},
  doi = {10.1109/ISIT.2010.5513563},
}

@article{koren2009matrix,
  title={Matrix factorization techniques for recommender systems},
  author={Koren, Yehuda and Bell, Robert and Volinsky, Chris},
  journal={Computer},
  volume={42},
  number={8},
  pages={30--37},
  year={2009},
  publisher={IEEE}
}

@InProceedings{Lee20,
  title   = {Logsmooth Gradient Concentration and Tighter Runtimes for {M}etropolized {H}amiltonian {M}onte {C}arlo},
  author  =       {Lee, Yin Tat and Shen, Ruoqi and Tian, Kevin},
  booktitle = 	 {Proceedings of the Thirty Third Conference on Learning Theory},
  pages = 	 {2565--2597},
  year = 	 {2020},
  editor = 	 {Abernethy, Jacob and Agarwal, Shivani},
  volume = 	 {125},
  series = 	 {Proceedings of Machine Learning Research},
  month = 	 {09--12 Jul},
  publisher =    {PMLR},
  url = 	 {https://proceedings.mlr.press/v125/lee20b.html},
}

@inproceedings{LeeShenTian2021,
  author    = {Lee, Yin Tat and Shen, Ruoqi and Tian, Kevin},
  title     = {Structured logconcave sampling with a restricted {G}aussian oracle},
  booktitle = {Proceedings of the Thirty Fourth Conference on Learning Theory},
  editor    = {Belkin, Mikhail and Kpotufe, Samory},
  series    = {Proceedings of Machine Learning Research},
  volume    = {134},
  pages     = {2993--3050},
  year      = {2021},
  publisher = {PMLR},
  url       = {https://proceedings.mlr.press/v134/lee21a.html}
}

@inproceedings{LiangChen2022,
  author    = {Liang, Jiaming and Chen, Yongxin},
  title     = {A proximal algorithm for sampling from non-smooth potentials},
  booktitle = {Proceedings of the 2022 Winter Simulation Conference},
  pages     = {3229--3240},
  year      = {2022}
}

@article{LiuTong2026,
  author  = {Liu, Shuigen and Tong, Xin T.},
  title   = {Spectral gap of {M}etropolis algorithms for non-smooth distributions under isoperimetry},
  journal = {Japanese Journal of Statistics and Data Science},
  year    = {2026},
  doi     = {10.1007/s42081-026-00354-0},
  url     = {https://doi.org/10.1007/s42081-026-00354-0}
}

@article{MaiAlquier2015,
  author  = {Mai, The Tien and Alquier, Pierre},
  title   = {A {B}ayesian approach for noisy matrix completion: optimal rate under general sampling distribution},
  journal = {Electronic Journal of Statistics},
  volume  = {9},
  number  = {1},
  pages   = {823--841},
  year    = {2015}
}

@article{MengWong1996,
  author  = {Meng, Xiao-Li and Wong, Wing Hung},
  title   = {Simulating ratios of normalizing constants via a simple identity: a theoretical exploration},
  journal = {Statistica Sinica},
  volume  = {6},
  pages   = {831--860},
  year    = {1996}
}

@article{Mou22,
  author  = {Wenlong Mou and Nicolas Flammarion and Martin J. Wainwright and Peter L. Bartlett},
  title   = {An Efficient Sampling Algorithm for Non-smooth Composite Potentials},
  journal = {Journal of Machine Learning Research},
  year    = {2022},
  volume  = {23},
  number  = {233},
  pages   = {1--50},
  url     = {http://jmlr.org/papers/v23/20-527.html}
}

@article{Neal2001,
  author  = {Neal, Radford M.},
  title   = {Annealed importance sampling},
  journal = {Statistics and Computing},
  volume  = {11},
  pages   = {125--139},
  year    = {2001}
}

@article{Negabhan2010,
  author = {Sahand Negahban and Martin J. Wainwright},
  title = {Restricted strong convexity and weighted matrix completion: Optimal bounds with noise},
  journal = {Journal of Machine Learning Research},
  year = {2012},
  volume = {13},
  number = {53},
  pages = {1665-1697},
  url = {https://jmlr.org/papers/v13/negahban12a.html},
}

@article{ParkCasella2008,
  author  = {Park, Trevor and Casella, George},
  title   = {The {B}ayesian lasso},
  journal = {Journal of the American Statistical Association},
  volume  = {103},
  number  = {482},
  pages   = {681--686},
  year    = {2008}
}

@article{Roman2012,
    author = {Jorge Carlos Rom{\'a}n and James P. Hobert},
    title = {{Convergence analysis of the Gibbs sampler for Bayesian general linear mixed models with improper priors}},
    volume = {40},
    journal = {The Annals of Statistics},
    number = {6},
    publisher = {Institute of Mathematical Statistics},
    pages = {2823 -- 2849},
    year = {2012},
    doi = {10.1214/12-AOS1052},
    URL = {https://doi.org/10.1214/12-AOS1052}
}

@article{Rosenthal1995,
  author = {Jeffrey S. Rosenthal},
  title = {Rates of convergence for {G}ibbs sampling for variance component models},
  journal = {The Annals of Statistics},
  year = {1995},
  month = {6},
  volume = {23},
  number = {3},
  pages = {740-761},
  doi = {10.1214/aos/1176324619},
}

@inproceedings{SalakhutdinovMnih2008,
  author    = {Salakhutdinov, Ruslan and Mnih, Andriy},
  title     = {{B}ayesian probabilistic matrix factorization using {M}arkov chain {M}onte {C}arlo},
  booktitle = {Proceedings of the 25th International Conference on Machine Learning},
  pages     = {880--887},
  year      = {2008}
}

@inproceedings{Richtarik19,
 author = {Salim, Adil and Kovalev, Dmitry and Richtarik, Peter},
 booktitle = {Advances in Neural Information Processing Systems},
 editor = {H. Wallach and H. Larochelle and A. Beygelzimer and F. d\textquotesingle Alch\'{e}-Buc and E. Fox and R. Garnett},
 pages = {},
 publisher = {Curran Associates, Inc.},
 title = {Stochastic Proximal {L}angevin Algorithm: Potential Splitting and Nonasymptotic Rates},
 url = {https://proceedings.neurips.cc/paper_files/paper/2019/file/6a8018b3a00b69c008601b8becae392b-Paper.pdf},
 volume = {32},
 year = {2019}
}

@inproceedings{Richtarik20,
 author = {Salim, Adil and Richtarik, Peter},
 booktitle = {Advances in Neural Information Processing Systems},
 editor = {H. Larochelle and M. Ranzato and R. Hadsell and M.F. Balcan and H. Lin},
 pages = {3786--3796},
 publisher = {Curran Associates, Inc.},
 title = {Primal Dual Interpretation of the Proximal Stochastic Gradient {L}angevin Algorithm},
 url = {https://proceedings.neurips.cc/paper_files/paper/2020/file/2779fda014fbadb761f67dd708c1325e-Paper.pdf},
 volume = {33},
 year = {2020}
}

@article{Tibshirani1996,
    url = {http://www.jstor.org/stable/2346178},
    author = {Robert Tibshirani},
    journal = {Journal of the Royal Statistical Society. Series B (Methodological)},
    number = {1},
    pages = {267--288},
    publisher = {[Royal Statistical Society, Oxford University Press]},
    title = {Regression Shrinkage and Selection via the Lasso},
    volume = {58},
    year = {1996}
}

@article{Xie2011,
  author  = {Xie, Wangang and Lewis, Paul O. and Fan, Yu and Kuo, Lynn and Chen, Ming-Hui},
  title   = {Improving marginal likelihood estimation for {B}ayesian phylogenetic model selection},
  journal = {Systematic Biology},
  volume  = {60},
  number  = {2},
  pages   = {150--160},
  year    = {2011}
}

\clearpage
\appendix
\thispagestyle{empty}

\section{Chi-squared mixing bounds for high-accuracy samplers}\label{app:A}
\subsection{Setup}\label{sub:setup}
We want to sample from the following measure on $\mathbb R^d$:
\[\pi(\mathrm dx) \propto \exp(-V(x))\,\mathrm{d}x, \qquad
V(x) = F(x) + G(x),\]
where $F$ is $m$-strongly convex and $L_F$-smooth
and $G$ is convex but nonsmooth and not everywhere differentiable,
but it is $L_G$-Lipschitz, so its minimal subgradient satisfies $\|\nabla^0 G\| \leq L_G$.
This problem is more difficult than the classic smooth-and-strongly-convex setting.
We will sample using a Metropolis-adjusted Markov chain $P$ targeting $\pi$, and we want a bound of the form
\[\chi^2(\pi_0 P^K \| \pi) \leq \varepsilon_{\mathrm{mix}}, \quad
K = \mathrm{poly}(d,\log(1/\varepsilon_{\mathrm{mix}})).\]
We will lean heavily on \cite{Andrieu24} and begin by giving an overview of the terminology used in that paper.
\begin{itemize}
    \item For a $\pi$-measurable set $A$, the \textbf{conductance}
    \[\Phi(A) = \frac{1}{\pi(A)} \int_A P(\bm{L}, A^c) \pi(\mathrm{d}\bm{L})\]
    measures how much probability mass flows between $A$ and $A^c$ per step.
    \item For a set $A$ with $\pi(A)=p$, its \textbf{boundary measure}
    \[\pi^+(A) := \liminf_{r \to 0}
    \left\{ \frac{\pi(B_r(A)) - \pi(A)}{r} \right\}\]
    measures how much $\pi$-mass sits at the edge of $A$.
    \item The \textbf{isoperimetric profile}
    $I_\pi(p) := \inf\{\pi^+(A): \pi(A)=p\}$
    measures the smallest perimeter (boundary measure) of a set of probability $p$.
    Intuitively, if $I_\pi$ is large, the boundaries are large, so each set of a given size has many places to exit and the Markov chain will not be stuck on one side of any boundary.
    \item An \textbf{isoperimetric minorant} $\tilde I_\pi \leq I_\pi$ is any lower bound on $I_\pi$.
    \item A minorant is \textbf{regular} if it is symmetric about $p = 1/2$, continuous, and increasing on $(0,1/2]$.
    \item The Markov kernel $P$ is $(\delta,\varepsilon)$-\textbf{close-coupling} if any two starting points within Euclidean
    distance $\delta$ of each other have next-step distributions within $1-\varepsilon$ TV distance of each other.
    \item The \textbf{spectral gap}
    \[\gamma_P := 1 - \|P\|_{L_0^2}, \quad
    \|P\|_{L_0^2} = \sup_{\|f\|_2=1,\int f\,\mathrm{d}\pi=0} \|Pf\|_2,\]
    controls geometric decay of chi-squared divergence:
    \[\chi^2(\pi_0 P^k\|\pi) \leq (1-\gamma_P)^{2k} \chi^2(\pi_0\|\pi).\]
\end{itemize}
\subsection{The general recipe}
The main result of \cite{Andrieu24} (Theorem 18) states that if
\begin{itemize}
    \item $\pi$ has regular concave isoperimetric minorant $\tilde I_\pi$,
    \item $P$ is $(\delta,\varepsilon)$-close-coupling, positive, and $\pi$-reversible (Gaussian random-walk Metropolis is positive, see \cite{Baxendale2005}), and
    \item $u_0:=\chi^2(\pi_0\|\pi)$ is finite,
\end{itemize}
we have $\chi^2(\pi_0 P^k \| \pi) \leq \varepsilon_{\mathrm{mix}}$ for $\varepsilon_{\mathrm{mix}} \in (0,8)$ after
\begin{align}
K &\geq 2 + \frac{64}{\varepsilon^2}
\max\left\{ \log \left( \frac{u_0 v_*}{4}\right), 0 \right\}
\tag*{(term 1: burn-in)}\\
&\qquad +\, \frac{256}{\varepsilon^2 \delta^2}
\int_{\max\{\min\{2/u_0,1/4\},v_*/2\}}^{1/4}
\frac{\xi}{\tilde I_\pi(\xi)^2}\mathrm d\xi
\tag*{(term 2: bulk mixing)}\\
&\qquad +\, \frac{16}{\varepsilon^2}
\max\left\{1,\frac{1}{4\delta^2 \tilde I_\pi(1/4)^2}\right\}
\max\left\{ \log \left( \frac{\min\{u_0,8\}}{\varepsilon_{\mathrm{mix}}} \right), 0 \right\}
\tag*{(term 3: fine convergence)}
\end{align}
where $v_*$ is a threshold scale solving
\[v_* := \min\left\{ \frac{1}{2}, \max\left\{0, \sup\left\{
v > 0: \tilde I_\pi\left(\frac{v}{2}\right) \geq \frac{v}{\delta}
\right\} \right\} \right\}.\]
All that remains is to find $\delta, \varepsilon, \tilde I_\pi$ for our specific setting.
To do this, we use another result from the paper: for any Metropolis kernel with proposal $Q$, we have
\[\|P(x,\cdot)-P(y,\cdot)\|_{\mathrm{TV}}
\leq \|Q(x,\cdot)-Q(y,\cdot)\|_{\mathrm{TV}} + 1-\alpha_0, \quad
\alpha_0 := \inf_x \alpha(x),\]
where $\alpha(x)$ is the average acceptance probability at $x$.
This splits the close-coupling problem into two independent pieces:
\begin{enumerate}
    \item How stable is the proposal as we move the current point?
    \item How low can the acceptance probability get?
\end{enumerate}
We will solve each separately and add the bounds.
\subsection{The target-side ingredient}\label{sub:target}
For any $m$-strongly convex potential, Caffarelli's contraction theorem (\cite{Caffarelli2000}) gives a $1$-Lipschitz map pushing $\mathcal{N}(0,m^{-1}\bm{I})$ onto $\pi$, so $\pi$ inherits the Gaussian isoperimetric profile, scaled by $\sqrt{m}$ (\cite{Andrieu24}, Lemma 27):
\[I_\pi(\xi) \geq \tilde I_\pi(\xi) := C_\ell \,\xi\sqrt{m\log(1/\xi)},
\quad \xi \in (0,1/2), \quad C_\ell = 0.958357.\]
This is dimension-free!
\subsection{The algorithm-side ingredients}\label{sub:ing}
Andrieu et al.'s acceptance-probability lemma (their Lemma 39) is stated for differentiable $V$ and requires a one-sided growth bound
\[V(\bm{x}+\bm{h}) - V(\bm{x}) - \langle \nabla V(\bm{x}),\bm{h} \rangle
\leq \psi(\|\bm{h}\|),\]
for some nondecreasing $\psi$. We therefore have
\begin{align*}
&V(\bm{x} + \bm{h}) - V(\bm{x})
- \langle \nabla^0 V(\bm{x}), \bm{h} \rangle &&\\
&\quad = F(\bm{x} + \bm{h}) - F(\bm{x})
- \langle \nabla F(\bm{x}), \bm{h} \rangle
+ G(\bm{x} + \bm{h}) - G(\bm{x})
- \langle \nabla^0 G(\bm{x}), \bm{h} \rangle &&\\
&\quad \leq \frac{L_F}{2}\|\bm h\|^2 + 2L_G\|\bm h\| &&\\
&\quad =: \psi(\|\bm h\|),
\end{align*}
where $\nabla^0 V := \nabla F + \nabla^0 G$.
A subgradient in the linear term suffices, so the conclusion of their Lemma 39 holds without differentiability of $V$. This observation is due to \cite{LiuTong2026}, who use the same growth function and the same symmetry argument to bound spectral gaps of random-walk Metropolis and the Metropolis-adjusted Langevin algorithm for smooth-plus-Lipschitz targets. We include a short proof in our notation for completeness.
\begin{lemma}\label{lem:accept}
Let $\alpha(\bm x)$ be the average acceptance probability at $\bm x$ of random-walk Metropolis (Algorithm~\ref{alg:rwm}) targeting $\pi$ with step size $\sigma$, and let $\psi(r) := \frac{L_F}{2}r^2 + 2L_G r$. For every $\bm x$, we have
\[\alpha(\bm x) \geq \frac{1}{2}\exp\left(-\mathbb{E}_{\bm{z}\sim \mathcal{N}(0,\bm{I})}[\psi(\sigma\|\bm z\|)]\right).\]
\end{lemma}
\begin{proof}
Fix $\bm x$ and define $\bm g := \nabla^0 V(\bm x)$ and $a(\bm h) := V(\bm x + \bm h) - V(\bm x)$.
The growth bound above gives $a(\bm h) \leq \langle \bm g, \bm h \rangle + \psi(\|\bm h\|)$, so whenever $\langle \bm g, \bm h \rangle \leq 0$, we have $a(\bm h) \leq \psi(\|\bm h\|)$. Since $\psi \geq 0$,
\[\min\{1, \exp(-a(\bm h))\} \geq \exp(-\psi(\|\bm h\|)).\]
With $\bm h = \sigma \bm z$ and $\bm z \sim \mathcal{N}(0,\bm I)$,
\[\alpha(\bm x) = \mathbb{E}[\min\{1, \exp(-a(\sigma\bm z))\}]
\geq \mathbb{E}\left[\exp(-\psi(\sigma\|\bm z\|))\,\mathbbm{1}\{\langle \bm g, \bm z \rangle \leq 0\}\right].\]
Because $\bm{z}$'s distribution is symmetric, we have
\[\mathbb{E}\left[\exp(-\psi(\sigma\|\bm z\|))\,\mathbbm{1}\{\langle \bm g, \bm z \rangle \leq 0\}\right]
\geq \frac{1}{2}\mathbb{E}[\exp(-\psi(\sigma\|\bm z\|))]
\geq \frac{1}{2}\exp(-\mathbb{E}[\psi(\sigma\|\bm z\|)])\]
by Jensen's inequality.
\end{proof}
For random-walk Metropolis, the proposal is
$\mathcal{N}(\bm{x},\sigma^2 \bm{I})$.
By the Gaussian KL identity and Pinsker's inequality,
\[\|Q_x-Q_y\|_{\mathrm{TV}}
\leq \sqrt{\frac{1}{2} \mathrm{KL}(Q_x \| Q_y)}
= \frac{\|x-y\|}{2\sigma}.\]
By Lemma~\ref{lem:accept},
\[\alpha_0 \geq \frac{1}{2} \exp(-\mathbb{E}_{\bm{z}\sim \mathcal{N}(0,\bm{I})} [\psi(\sigma\|\bm z\|)]).\]
Plugging in our $\psi$ and choosing
\[\sigma = \frac{1}{4\sqrt{d}\,M}, \quad M := \max\{L_G, \sqrt{L_F}\}\]
makes both terms of $\mathbb{E}[\psi(\sigma\|\bm z\|)]$ bounded, dimension-free constants.
Since $\mathbb{E}\|\bm z\|^2 = d$ and $\mathbb{E}\|\bm z\| \leq \sqrt{d}$,
\[\mathbb{E}[\psi(\sigma\|\bm z\|)]
\leq \frac{L_F \sigma^2 d}{2} + 2L_G\sigma\sqrt{d}
= \frac{L_F}{32M^2} + \frac{L_G}{2M}
\leq \frac{17}{32},\]
so $\alpha_0 \geq \frac{1}{2} \exp(-\frac{17}{32}) \approx 0.29393$.
\subsection{Assembling the mixing bound}\label{sub:mixing}
We will evaluate the integral from term 2 explicitly.
We find that the equation $\tilde I_\pi(v/2) = v/\delta$ is solved by
\[v_* = 2\exp\left(-\frac{4}{C_\ell^2 m \delta^2}\right),\]
which will be squared-exponentially small if $\delta$ is small.
If $u_0 < 4/v_*$, which we verify for our application at the end of Appendix \ref{sub:bound}, the first term is zero.
In that case, since every term of the bound is nondecreasing in $u_0$, the bound remains valid with $u_0$ replaced by $\bar u_0 := \max\{u_0, 8\}$, so we may and do assume $u_0 \ge 8$. 
The integral in the second term then comes out to
\[\int_{2/u_0}^{1/4} \frac{\xi}{\tilde I_\pi(\xi)^2}\,\mathrm d\xi
= \frac{1}{C_\ell^2m} \int_{2/u_0}^{1/4}
\frac{1}{\xi\log(1/\xi)} \,\mathrm d\xi
= \frac{\log\log(u_0/2) - \log\log(4)}{C_\ell^2m}.\]
Our bound is now
\begin{align*}
K &\geq 2 + \frac{256}{\varepsilon^2 \delta^2}
\cdot \frac{\log\log(u_0/2) - \log\log(4)}{C_\ell^2 m}
+ \frac{64}{\varepsilon^2}
\cdot \frac{1}{C_\ell^2 m \delta^2 \log(4)}
\log \left( \frac{8}{\varepsilon_{\mathrm{mix}}} \right) &&\\
&= 2 + \frac{64}{C_\ell^2 m \varepsilon^2 \delta^2}
\left[ 4(\log\log(u_0/2) - \log\log(4))
+ \frac{1}{\log(4)}
\log \left( \frac{8}{\varepsilon_{\mathrm{mix}}} \right) \right].
\end{align*}
Now it remains to choose $\varepsilon$ and $\delta$.
Combining the earlier results, we see that
\[\|P(x,\cdot)-P(y,\cdot)\|_{\mathrm{TV}}
\leq \|Q(x,\cdot)-Q(y,\cdot)\|_{\mathrm{TV}} + 1 - \alpha_0
\leq \frac{\|x-y\|}{2\sigma} + 1 - \alpha_0.\]
To turn this into a $(\delta,\varepsilon)$-close-coupling,
we need a bound $\delta$ on the Euclidean distance and a bound $1-\varepsilon$ on the TV distance.
We choose $\delta = \alpha_0 \sigma$ so that the first term is at most half of what $\alpha_0$ allows.
The resulting right side of the inequality above tells us to set $\varepsilon = \alpha_0/2$.
Not knowing $\alpha_0$, we substitute the lower bound from Appendix \ref{sub:ing}, and we get
\[\delta = \frac{1}{8\sqrt{d}\,M} \exp\left(-\frac{17}{32}\right), \quad
\varepsilon = \frac{1}{4} \exp\left(-\frac{17}{32}\right), \quad
\varepsilon^2\delta^2 = \frac{1}{1024dM^2} \exp\left(-\frac{17}{8}\right).\]
Because $m \leq L_F \leq M^2$, we have $m\delta^2 < 1/d$, so 
\[4\delta^2 \tilde I_\pi(1/4)^2 = \frac{C_\ell^2 m \delta^2 \log 4}{4} < 1,\] 
and the maximum in term 3 is attained by its second argument, as used above.
Our final bound is
\[K \geq 2 + 65536 \cdot \frac{d \max\{L_G^2, L_F\}}{C_\ell^2 m}
\exp\left(\frac{17}{8}\right)
\left[ 4(\log\log(u_0/2) - \log\log(4))
+ \frac{1}{\log(4)}
\log \left( \frac{8}{\varepsilon_{\mathrm{mix}}} \right) \right],\]
with no factor exponential in $L_F/L_G^2$. The step size $\sigma \propto 1/(\sqrt{d}\max\{L_G,\sqrt{L_F}\})$ matches the scaling of the step-size condition of \cite{LiuTong2026}.
For $\mu_t$, we have $d = n_1 n_2$, $m = 2/B$, $L_F = 2/B + t$, and $L_G = \sqrt{n_1}/\lambda$, which gives the values of $\sigma(t)$ and $K(t)$ in Table~\ref{tab:values} with $\varepsilon_{\mathrm{mix}} = \varepsilon_\star^2$.

\subsection{A bound on the initial chi-squared divergence}\label{sub:bound}
The previous sections' results hinge on $u_0$, which is difficult to compute and often difficult to bound. In our case, though, an obvious initialization measure exists with a tractable chi-squared bound.

\begin{lemma}\label{lem:X2}
Let
\[
\pi(\mathrm d\bm L)
\propto
\exp\!\left(
-\sum_{i,j}\phi_{ij}(L_{ij})
-\frac1\lambda\|\bm L\|_*
\right)\mathrm d\bm L,
\]
where each $\phi_{ij}:\mathbb R\to\mathbb R$ is $m$-strongly convex.
Define the product measure
\[\pi^0(\mathrm d\bm L) \propto \exp\!\left( -\sum_{i,j}\phi_{ij}(L_{ij}) \right)\mathrm d\bm L.\]
Then
\[\chi^2(\pi^0\|\pi)
\le \exp\!\left(\frac{\sqrt r}{\lambda} \sqrt{\frac{d}{m}
+ \|\mathbb E_{\pi^0}[\bm L]\|_{\mathrm F}^2}
+ \frac{r}{2m\lambda^2}\right)-1,\]
where $d = n_1n_2$ and $r=\min\{n_1,n_2\}$.
\end{lemma}
\begin{proof}
Let
\[Z_\pi = \int \exp\!\left(- \sum_{i,j} \phi_{ij}(L_{ij})
- \frac{1}{\lambda} \|\bm L\|_* \right)\mathrm d\bm L, \quad
Z_{\pi^0} = \int \exp\!\left(-\sum_{i,j}\phi_{ij}(L_{ij}) \right)\mathrm d\bm L.\]
Since the nuclear norm is nonnegative, $Z_\pi \le Z_{\pi^0}$. Moreover,
\[\frac{\mathrm d\pi^0}{\mathrm d\pi}(\bm L)
= \frac{Z_\pi}{Z_{\pi^0}} \exp\!\left(\frac1\lambda\|\bm L\|_*\right),\]
so
\[\chi^2(\pi^0\|\pi) =
\int \left(\frac{\mathrm d\pi^0}{\mathrm d\pi} \right)^2 \mathrm d\pi -1
= \frac{Z_\pi}{Z_{\pi^0}} \mathbb E_{\pi^0}
\!\left[\exp\!\left(\frac1\lambda\|\bm L\|_* \right)\right] -1
\leq \mathbb E_{\pi^0} \!\left[\exp\!\left(\frac{1}{\lambda}
\|\bm L\|_*\right)\right] -1.\]
Using $\|\bm L\|_* \le \sqrt r\,\|\bm L\|_{\mathrm F}$,
we obtain
\[\chi^2(\pi^0\|\pi) \le \mathbb E_{\pi^0} \!\left[
\exp\!\left(\frac{\sqrt r}{\lambda}\|\bm L\|_{\mathrm F}\right)
\right] -1.\]
Since each $\phi_{ij}$ is $m$-strongly convex, the product measure $\pi^0$ is $m$-strongly log-concave.
The function $f(\bm L)=\|\bm L\|_{\mathrm F}$
is $1$-Lipschitz with respect to the Frobenius norm.
The Gaussian concentration inequality for strongly log-concave measures gives
\[\mathbb E_{\pi^0} \!\left[\exp\!\left(\frac{\sqrt r}{\lambda}f\right)\right]
\leq \exp\!\left( \frac{\sqrt r}{\lambda}\,\mathbb E_{\pi^0}[f]
+ \frac{r}{2m\lambda^2}\right).\]
By the Brascamp--Lieb inequality,
\[\mathbb{V}[L_{ij}] \le \frac1m,\]
so
\[\mathbb E_{\pi^0}[\|\bm L\|_{\mathrm F}^2]
\le \frac dm + \|\mathbb E_{\pi^0}[\bm L]\|_{\mathrm F}^2.\]
By Jensen's inequality,
\[\mathbb E_{\pi^0}[\|\bm L\|_{\mathrm F}]
\le \sqrt{\frac dm + \|\mathbb E_{\pi^0}[\bm L]\|_{\mathrm F}^2}.\]
Therefore,
\[\mathbb E_{\pi^0}\!\left[\exp\!\left(\frac{\sqrt r}{\lambda}f\right)\right]
\le \exp\!\left(\frac{\sqrt r}{\lambda}
\sqrt{\frac dm + \|\mathbb E_{\pi^0}\bm L\|_{\mathrm F}^2}
+ \frac{r}{2m\lambda^2} \right),\]
and our desired result follows.
\end{proof}
\paragraph{Verification that $u_0 < 4/v_*$.} With the choices of Appendix \ref{sub:mixing}, $4/v_* = 2\exp(4/(C_\ell^2 m \delta^2))$ and $\delta^2 = e^{-17/16}/(64 d M^2)$, so $u_0 < 4/v_*$ whenever
\[\log(u_0/2) < \frac{256\,e^{17/16}\,d M^2}{C_\ell^2 m}.\]
For $\mu_t$, let $x := \sqrt{n_1}/\lambda$, so that $M^2 = \max\{x^2, 2/B + t\}$ and $m = 2/B$. By Lemma~\ref{lem:X2}, as in Table~\ref{tab:values} and the AM--GM inequality $x\sqrt{A(t)} \leq \frac{Bd}{4}x^2 + \frac{A(t)}{Bd}$,
\[\log(u_0/2) < \log(1+u_0) \leq x\sqrt{A(t)} + \frac{B}{4}x^2
\leq \frac{Bd}{2}x^2 + \frac{1}{2} + \frac{\|\bm{Y\circ\Omega}\|_{\mathrm F}^2}{Bd}
\leq \frac{5}{2}\max\left\{\frac{Bd}{2}x^2, d\right\}\]
whenever $\|\bm{Y\circ\Omega}\|_{\mathrm F}^2 \leq B d^2$, which holds under the standardization of Corollary~\ref{cor:TV} once $n_1 n_2$ exceeds a constant. Since $dM^2/m = \frac{Bd}{2}\max\{x^2, \frac{2}{B} + t\} \geq \max\{\frac{Bd}{2}x^2, d\}$ and $256e^{17/16}/C_\ell^2 > 800$, the condition holds.
\subsection{A practical substitute: Metropolis-adjusted Langevin}\label{app:mala}
For $V = F + G$ as in Appendix \ref{sub:setup}, let $\nabla^0 V := \nabla F + \nabla^0 G$. The minimal subgradient of the nuclear norm at a point $\bm{L}$ with rank $r$ and compact SVD $\bm{U}_r\bm{\Sigma}_r\bm{V}_r^\top$ is $\bm{U}_r\bm{V}_r^\top$, so
\[\nabla^0 V(\bm{L}) = \frac{2}{B}\bm{L} + t(\bm{L}-\bm{Y})\circ\bm{\Omega} + \frac{1}{\lambda}\bm{U}_r\bm{V}_r^\top.\]
Since $\mu_t$ has a density, $\bm{L}$ has full rank almost surely, and $\nabla^0 V$ is then the gradient of $V$.
\begin{algorithm}
\caption{Metropolis-adjusted Langevin for $\pi \propto \exp(-V)$}
\label{alg:mala}
\begin{algorithmic}[1]
\Require step size $h > 0$; number of iterations $K$; initial distribution $\pi^0$
\State Define $q(\bm{Z}\mid\bm{X}) \propto \exp\left(-\frac{1}{4h}\|\bm{Z} - \bm{X} + h\nabla^0 V(\bm{X})\|_{\mathrm F}^2\right)$
\State $\bm{L}^0 \sim \pi^0$
\For{$k = 0, 1, \dots, K-1$}
    \State $\bm{\xi}^k \sim \mathcal{N}(0,\bm{I})$
    \State $\bm{Z}^k \gets \bm{L}^k - h\nabla^0 V(\bm{L}^k) + \sqrt{2h}\,\bm{\xi}^k$
    \Comment{proposal}
    \State $U^k \sim \mathrm{Unif}(0,1)$
    \If{$\log U^k \leq V(\bm{L}^k) - V(\bm{Z}^k) + \log q(\bm{L}^k\mid\bm{Z}^k) - \log q(\bm{Z}^k\mid\bm{L}^k)$}
        \State $\bm{L}^{k+1} \gets \bm{Z}^k$
    \Else
        \State $\bm{L}^{k+1} \gets \bm{L}^k$
    \EndIf
\EndFor
\State \Return $\bm{L}^K$
\end{algorithmic}
\end{algorithm}

\paragraph{Guarantee.} For $m$-strongly convex $V$, \cite{LiuTong2026} show that the spectral gap of Algorithm~\ref{alg:mala} is 
\[\Omega(mh) \quad \text{for} \quad 
h = O\!\left(\min\left\{\frac{m}{dL_F^2}, \frac{1}{dL_G^2}\right\}\right)\] 
whereas the spectral gap of random-walk Metropolis is 
\[\Omega(m\sigma^2) \quad \text{for} \quad 
\sigma = O\!\left(\min\left\{\frac{1}{\sqrt{dL_F}}, \frac{1}{\sqrt{d}\,L_G}\right\}\right).\] 
At the largest allowed steps, both gaps are of order $m/(dL_G^2)$ whenever $L_G^2 \ge L_F^2/m$. For $\mu_t$, $L_G^2 = n_1/\lambda^2$ and $L_F^2/m = (1 + tB/2)(2/B + t)$, so under the assumptions of Corollary~\ref{cor:TV}, this is the regime $a < 1/2$ with $n_1$ large. When $L_G^2 < L_F^2/m$, the gap for the Metropolis-adjusted Langevin algorithm is smaller than that for random-walk Metropolis by a factor of up to $L_F/m$.

A spectral gap $\gamma$ of $P$ gives $\chi^2(\pi_0\bar P^k\|\pi) \leq (1-\gamma/2)^{2k}u_0$ for the lazy chain $\bar P := (I+P)/2$, so $K = O(\gamma^{-1}\log(u_0/\varepsilon_{\mathrm{mix}}))$ steps suffice. Compared with the bound of Appendix A.5, this replaces $\log\log u_0$ by $\log u_0$, which for $\mu_t$ is of order $d\max\{n_1/\lambda^2, 1\}$ (end of Appendix A.6), so the resulting guarantee is weaker by a polynomial factor. 

\section{Some moment bounds and transport inequalities}
Before we introduce the following lemmas, it will be useful to define the following for a measure $\rho$:
\begin{itemize}
    \item $Z_{\rho}(s) := \mathbb{E}_{\rho}\left[ \exp\left(-\frac{s}{2} R(\bm{L}) \right) \right]$
    is the scaling constant of $\rho_s$.
    \item $\Lambda_{\rho}(s) := \log (s^{N/2} Z_{\rho}(s))$
    is the part of the log-posterior which depends on $s$.
    \item $\widehat{\Lambda}_{\rho}(I) = \log \int_I t^{N/2-1} Z_{\rho}(t)\,\mathrm{d}t$
    is the log-averaged version of $\Lambda_{\rho}(s)$ over $I$.
    \item $H_{\rho}(s) = \exp\left( \frac{sN}{2} \right) Z_{\rho}(e^s)$
    is $\tau$ on a logarithmic scale,
    so $s_q := \log \tau_q$ is the midpoint of the interval $\log I_q$,
    which has width $\Delta_s := (\log \kappa)/Q$.
\end{itemize}
\begin{lemma}\label{lem:ER}
    We have $\mathbb{E}_{\mu}[R(\bm{L})] \leq R_1$.
\end{lemma}
\begin{proof}
Because $\mu$ is $2/B$-strongly log-concave, the Brascamp-Lieb inequality gives
\[\operatorname{Cov}_{\mu}(\bm{L}) \preceq \frac{B}{2} \bm{I}.\]
Taking traces,
\[\mathbb{E}_{\mu} [\|\bm{L \circ \Omega}\|_{\mathrm{F}}^2]
\leq \frac{NB}{2}.\]
Because $\bm{L}$ has mean zero under $\mu$ and $\bm{Y \circ \Omega}$ is deterministic, we have
\[\mathbb{E}_{\mu} [\langle \bm{L \circ \Omega}, \bm{Y \circ \Omega} \rangle_{\mathrm{F}}] = 0,\]
so
\[\mathbb{E}_{\mu} [\|\bm{(L-Y) \circ \Omega}\|_{\mathrm{F}}^2]
= \|\bm{Y \circ \Omega}\|_{\mathrm{F}}^2
+ \mathbb{E}_{\mu} [\|\bm{L \circ \Omega}\|_{\mathrm{F}}^2]
\leq \|\bm{Y \circ \Omega}\|_{\mathrm{F}}^2 + \frac{NB}{2}.\]
\end{proof}
\begin{lemma}\label{lem:ER4}
    We have $\mathbb{E}_{\mu}[R^2(\bm{L})] \leq R_2$.
\end{lemma}
\begin{proof}
Let $g(\bm{L}) := \|\bm{L \circ \Omega}\|_{\mathrm{F}}^2$. As in the proof of Lemma~\ref{lem:ER}, $\mathbb{E}_{\mu}[g] \leq NB/2$. Since $\|\nabla g\|^2 = 4g$, the Poincar\'e inequality for the $2/B$-strongly log-concave measure $\mu$ gives
\[\mathbb{V}_{\mu}[g] \leq \frac{B}{2}\mathbb{E}_{\mu}[\|\nabla g\|^2] = 2B\,\mathbb{E}_{\mu}[g] \leq NB^2,\]
so
\[\mathbb{E}_{\mu} [\|\bm{L \circ \Omega}\|_{\mathrm{F}}^4]
= \mathbb{E}_{\mu}[g]^2 + \mathbb{V}_{\mu}[g]
\leq \frac{N(N+4)B^2}{4}.\]
Using Minkowski in $L^4$,
\begin{align*}
\mathbb{E}_{\mu}[\|\bm{(L-Y) \circ \Omega}\|_{\mathrm{F}}^4]
&\leq \left(\|\bm{Y \circ \Omega}\|_{\mathrm{F}}
+ \left(\mathbb{E}_{\mu}[\|\bm{L \circ \Omega}\|_{\mathrm{F}}^4] \right)
^{\frac{1}{4}} \right)^4 &&\\
&\leq \left(\|\bm{Y \circ \Omega}\|_{\mathrm{F}}
+ \left(\frac{N(N+4)B^2}{4} \right)
^{\frac{1}{4}} \right)^4.
\end{align*}
\end{proof}

\begin{lemma}\label{lem:Q}
    We have
    \begin{align*}
    W_1(\nu_{\star},\nu^Q_{\star})
    &\leq \tau_{\max} (\kappa^{1/Q}-1)
    + \frac{(\tau_Q - \tau_1) \log^2 \kappa}{16Q^2}
    \Bigg( \left( \frac{N}{2} + \frac{\tau_{\max} R_1}{2}
    \right)^2
    + \frac{\tau_{\max} R_1}{6}
    + \frac{\tau_{\max}^2 R_2}{12} \Bigg).
    \end{align*}
\end{lemma}
\begin{proof}
Letting $\kappa = \tau_{\max}/\tau_{\min}$, we have
$\tau_q = \tau_{\min} \kappa^{(q-1/2)/Q}$.
Let \[I_q = [\tau_{\min} \kappa^{(q-1)/Q}, \tau_{\min} \kappa^{q/Q}]\]
and define the categorical distribution $\widehat{\nu}^Q_{\star}$ such that $\widehat{\nu}^Q_{\star}(\tau_q) = \nu_{\star}(I_q)$.
By the triangle inequality,
\[W_1(\nu_{\star},\nu^Q_{\star})
\leq W_1(\nu_{\star},\widehat{\nu}^Q_{\star})
+ W_1(\widehat{\nu}^Q_{\star},\nu^Q_{\star}).\]
Using the simple coupling from $\nu_{\star}$ to $\widehat{\nu}^Q_{\star}$ that maps a sample $\tau \in I_q$ to $\tau_q$, we can see
\[W_1(\nu_{\star},\widehat{\nu}^Q_{\star})
\leq \tau_{\max} (\kappa^{1/Q}-1).\]
If we let $\psi = \log H_{\mu}$ and $U \sim \mathrm{Unif}[-\Delta_s/2,\Delta_s/2]$,
we have
\[\widehat{\Lambda}_{\mu}(I_q)
= \log \left(\int_{s_q-\Delta_s/2}^{s_q+\Delta_s/2} H_{\mu}(s) \mathrm{d}s \right)
= \Lambda_{\mu}(\tau_q) + \log \Delta_s + \log \mathbb{E}_U[\exp(\phi_q(U))]\]
for $\phi_q(u) := \psi(s_q+u) - \psi(s_q)$. By the mean value theorem,
\[|\phi_q(u)| \leq \|\psi'\|_{\infty} |u|
\leq \|\psi'\|_{\infty} \frac{\Delta_s}{2}.\]
By Hoeffding's lemma, we now have
\[\log \mathbb{E}_U[\exp(\pm \phi_q(U))]
\leq \pm \mathbb{E}_U[\phi_q(U)] + \frac{\|\psi'\|_{\infty}^2 \Delta_s^2}{8}.\]
Using the midpoint rule as an error bound,
\[|\mathbb{E}_U[\phi_q(U)]| = \frac{1}{\Delta_s}
\left| \int_{-\Delta_s/2}^{\Delta_s/2} (\psi(s_q+u) - \psi(s_q)) \mathrm{d}u \right|
\leq \frac{\|\psi''\|_{\infty} \Delta_s^2}{24},\]
so applying in the positive and negative directions, we get
\[|\log \mathbb{E}_U[\exp(\phi_q(U))]|
\leq \frac{\|\psi''\|_{\infty} \Delta_s^2}{24}
+ \frac{\|\psi'\|_{\infty}^2 \Delta_s^2}{8}.\]
If $t = e^s \in [\tau_{\min},\tau_{\max}]$, the computation at the start of Appendix~\ref{sec:ti-fix} gives
\[\psi'(s) = \frac{N}{2} - \frac{t}{2}\mathbb{E}_{\mu_t}[R], \qquad
\psi''(s) = -\frac{t}{2}\mathbb{E}_{\mu_t}[R] + \frac{t^2}{4}\mathbb{V}_{\mu_t}[R].\]
The same computation shows that 
\[\frac{\mathrm{d}}{\mathrm{d}t}\mathbb{E}_{\mu_t}[R] = -\frac{1}{2}\mathbb{V}_{\mu_t}[R] \leq 0, \quad 
\frac{\mathrm{d}}{\mathrm{d}t}\mathbb{E}_{\mu_t}[R^2] = -\frac{1}{2}\operatorname{Cov}_{\mu_t}(R^2,R) \leq 0,\] 
since $R^2$ and $R$ are nondecreasing functions of the same random variable $R \geq 0$. It follows that for all $t>0$, we have
\[\mathbb{E}_{\mu_t}[R] \leq \mathbb{E}_{\mu}[R], \quad 
\mathbb{V}_{\mu_t}[R] \leq \mathbb{E}_{\mu_t}[R^2] \leq \mathbb{E}_{\mu}[R^2],\] 
so by Lemmas~\ref{lem:ER} and~\ref{lem:ER4},
\[\|\psi'\|_{\infty}
\leq \frac{N}{2} + \frac{\tau_{\max}}{2} \mathbb{E}_{\mu}[R]
\leq \frac{N}{2} + \frac{\tau_{\max}}{2} R_1,\]
\[\|\psi''\|_{\infty}
\leq \frac{\tau_{\max}}{2} \mathbb{E}_{\mu}[R]
+ \frac{\tau_{\max}^2}{4} \mathbb{E}_{\mu}[R^2]
\leq \frac{\tau_{\max}}{2} R_1
+ \frac{\tau_{\max}^2}{4} R_2.\]
On a bounded set, we have
\[W_1(\widehat{\nu}^Q_{\star},\nu_{\star}^Q) \leq (\tau_Q - \tau_1)
\|\widehat{\nu}^Q_{\star}-\nu_{\star}^Q\|_{\mathrm{TV}}
\leq \frac{\tau_Q - \tau_1}{2} \max_q |\log \mathbb{E}_U[\exp(\phi_q(U))]|,\]
by shift-invariance of the softmax function,
implying our desired result.
\end{proof}

\subsection{The thermodynamic integration step}
\label{sec:ti-fix}
We first compute the score and curvature of $\Lambda_\mu$, which gives both a derivative bound and a direct proof of the monotonicity of $\mathbb{E}_{\nu_t}[R(\bm{L})]$ in $t$, used in the proof of Lemma~\ref{lem:Q}. Because
\[\mathbb{E}_{\mu_s}[R]
= \frac{\mathbb{E}_\mu[R\exp(-sR/2)]}{Z_\mu(s)},
\quad \frac{\mathrm{d}\mu_s}{\mathrm{d}\mu}
= \frac{\exp(-sR/2)}{Z_\mu(s)}
\,\Rightarrow\, \frac{Z_\mu'(s)}{Z_\mu(s)}
= -\frac12\mathbb{E}_{\mu_s}[R],\]
we have
\[\frac{\mathrm{d}}{\mathrm{d}s} \mathbb{E}_{\mu_s}[R]
= -\frac12\mathbb{E}_{\mu_s}[R^2] - \mathbb{E}_{\mu_s}[R] \frac{Z_\mu'(s)}{Z_\mu(s)}
= -\frac12\mathbb{E}_{\mu_s}[R^2] + \frac12\,\mathbb{E}_{\mu_s}[R]^2
= -\frac12\,\mathbb{V}_{\mu_s}[R].\]
Differentiating $\Lambda_\mu(s) = \frac{N}{2}\log s + \log Z_\mu(s)$
and using $\frac{\mathrm{d}}{\mathrm{d}s} \log Z_\mu(s)
= -\tfrac12\mathbb{E}_{\mu_s}[R]$, we obtain
\[\Lambda_\mu'(s) = \frac{N}{2s} - \frac12\mathbb{E}_{\mu_s}[R(L)],
\qquad
\Lambda_\mu''(s) = -\frac{N}{2s^2} + \frac14\mathbb{V}_{\mu_s}[R(L)].\]
Step 4 of the proof below will also need a uniform bound on the fourth central moment of $R$ under the tilted family. 
\begin{lemma}\label{lem:centered4th}
For all $s\ge0$,
\[\mathbb{E}_{\mu_s}[(R(\bm{L}) - \mathbb{E}_{\mu_s}[R(\bm{L})])^4]
\leq \widehat R_4.\]
\end{lemma}
\begin{proof}
Let $f(\bm{L}):=\sqrt{R(\bm{L})}$.
Since $\Omega$ has entries in $\{0,1\}$,
\[|f(\bm{L})-f(\bm{L}')|
\leq \|(\bm{L}-\bm{L}')\circ\bm{\Omega}\|_{\mathrm{F}}
\leq \|\bm{L}-\bm{L}'\|_{\mathrm{F}},\]
so $f$ is $1$-Lipschitz in Frobenius norm.
$\mu_s$ is $2/B$-strongly log-concave for every $s$,
so by Caffarelli's contraction theorem (Appendix \ref{sub:target}),
there is a $1$-Lipschitz map pushing $\mathcal{N}(0,\frac{B}{2}\bm{I})$ onto $\mu_s$. 
Composing with the $1$-Lipschitz $f$ and applying Gaussian concentration gives
\[\mu_s(|f-\mathbb{E}_{\mu_s}f|>t) \leq 2\exp(-t^2/B) \quad\forall t>0.\]
(This is the same concentration fact used in the proof of Lemma~\ref{lem:X2}, applied here to $f$ instead of $\|\bm{L}\|_{\mathrm{F}}$.)

\textbf{Moments of the centered part.} Let $g:=f-\bar{f}$, $\bar f:=\mathbb{E}_{\mu_s}[f]$.
For even $p\ge2$, tail-integrating the bound above gives
\[\mathbb{E}[|g|^p] = \int_0^\infty pt^{p-1}\mu_s(|g|>t)\,\mathrm{d}t
\leq 2p\int_0^\infty t^{p-1}e^{-t^2/B}\,\mathrm{d}t
= p\Gamma(p/2)B^{p/2}.\]
In particular, $\mathbb{E}[g^4] \leq 4\Gamma(2)B^2 = 4B^2$ and
$\mathbb{E}[g^8] \leq 8\Gamma(4)B^4 = 48B^4$.
Directly by Poincar\'e (as in Step~2 below, using $\|\nabla f\|=1$ a.e.),
\[\mathbb{E}[g^2] = \mathbb{V}_{\mu_s}[f]
\leq \frac{B}{2}\mathbb{E}_{\mu_s}[\|\nabla f\|^2] \leq \frac{B}{2}.\]
\textbf{From moments of $f$ to a centered moment of $R$.} Since $R=f^2=\bar f^2+2\bar fg+g^2$,
for $\bar R:=\mathbb{E}_{\mu_s}[R]=\bar f^2+\mathbb{E}[g^2]$,
\[R-\bar{R} = 2\bar{f}g + (g^2-\mathbb{E}[g^2]).\]
Using $(a+b)^4 \leq 8(a^4+b^4)$ twice, we have
\[(R-\bar{R})^4 \leq 8(16\bar{f}^4g^4 + (g^2-\mathbb{E}[g^2])^4), \quad
(g^2-\mathbb{E}[g^2])^4 \leq 8(g^8+(\mathbb{E}[g^2])^4).\]
Taking expectations,
\[\mathbb{E}[(R-\bar R)^4]
\leq 128\bar{f}^4 \mathbb{E}[g^4]
+ 64(\mathbb{E}[g^8]+(\mathbb{E}[g^2])^4)
= 512\bar{f}^4 B^2 + 3076B^4.\]
Finally, $\bar f^2\le\mathbb{E}_{\mu_s}[f^2]=\mathbb{E}_{\mu_s}[R]\le R_1$ by Jensen, monotonicity, and Lemma~\ref{lem:ER}, so $\bar{f}^4 \le R_1^2$, giving the claim.
\end{proof}
\begin{lemma}[Thermodynamic-integration bound]\label{lem:TI}
We have
\[\mathbb{E}\Big[W_1\big(\nu_\star^Q,
\nu_\star^{Q,\mathrm{TI}}\big)\Big]
\leq \frac{(\tau_Q-\tau_1)^2}{4}
\left(\frac{L_f\Delta}{2} + \sqrt{\widetilde R_2}\,\varepsilon_\star\right)
+ \frac{\tau_Q-\tau_1}{2}
\sqrt{\frac{\widehat{V} \Delta (\tau_Q-\tau_1)}{M}}.\]
\end{lemma}
\begin{proof}
\textbf{Step 0 (reduce to a sup-norm comparison).}
As in the proof of Lemma~\ref{lem:Q}, shift-invariance of the softmax gives
\[\nu_\star^Q = \mathrm{softmax}
(\{\bar\Lambda_\mu(\tau_1), \dots, \bar\Lambda_\mu(\tau_Q)\})\]
for
\[\bar\Lambda_\mu(\tau_q) := \Lambda_\mu(\tau_q)-\Lambda_\mu(\tau_1) = \int_{\tau_1}^{\tau_q} f(s)\,\mathrm ds, \quad
f(s) := \Lambda_\mu'(s) = \frac{N}{2s}-\frac12\mathbb{E}_{\mu_s}[R(\bm{L})].\]
Because the softmax function is $\frac12$-Lipschitz from the
infinity norm to TV distance, we can use the TV bound for W1 distance on a bounded interval to say that
\[W_1\big(\nu_\star^Q,\nu_\star^{Q,\mathrm{TI}}\big)
\le (\tau_Q-\tau_1)\,\big\|\nu_\star^Q-\nu_\star^{Q,\mathrm{TI}}\big\|_{\mathrm{TV}}
\le \frac{\tau_Q-\tau_1}{2}\,\max_q\big|\bar\Lambda_\mu(\tau_q)-\widetilde\Lambda_q\big| .\]
\textbf{Step 1 (exact three-way split).}
Fix $q$ and write $h_i := \tau_{i+1}-\tau_i$. The trapezoid rule of Algorithm~\ref{alg:ti} gives the grid point $\tau_i$, $i \leq q$, the weight
\[w_i^{(q)} := \frac{h_{i-1}\mathbbm{1}\{i>1\} + h_i\mathbbm{1}\{i<q\}}{2}, \qquad \sum_{i\leq q} w_i^{(q)} = \tau_q - \tau_1,\]
so that 
\[\widetilde\Lambda_q = \sum_{i\leq q} w_i^{(q)}\left(\frac{N}{2\tau_i} - \frac12\widehat R_M(\tau_i)\right).\] 
Since $f(\tau_i) = \tfrac{N}{2\tau_i}-\tfrac12\mathbb{E}_{\mu_{\tau_i}}[R]$, we have
\begin{align*}
\bar\Lambda_\mu(\tau_q)-\widetilde\Lambda_q
&= \left(\int_{\tau_1}^{\tau_q}\! f(s)\,\mathrm ds 
- \sum_{i<q} \frac{f(\tau_i)+f(\tau_{i+1})}{2}\,h_i\right)
\tag*{(term A: quadrature)} \\
&+ \frac12\sum_{i\leq q} w_i^{(q)}\big(\mathbb{E}_{\mu_{\tau_i}^{K(\tau_i)}}[R]-\mathbb{E}_{\mu_{\tau_i}}[R]\big)
\tag*{(term B: chain bias)} \\
&+ \frac12\sum_{i\leq q} w_i^{(q)}\big(\widehat R_{M}(\tau_i)-\mathbb{E}_{\mu_{\tau_i}^{K(\tau_i)}}[R]\big).
\tag*{(term C: MC noise)}
\end{align*}
Terms A and B are a deterministic quadrature error and a difference of fixed expectations respectively, so only term C is random, which matters in Step 4.
 
\textbf{Step 2 (term A).}
We already saw $f'(s) = -\frac{N}{2s^2}+\frac14\mathbb{V}_{\mu_s}[R]$.
By the Brascamp--Lieb/Poincar\'e inequality under $\mu_s$'s strong-convexity constant $2/B$ applied to $R$ itself, and using $\|\nabla R\|^2 = 4R$,
we have
\[\mathbb{V}_{\mu_s}[R] \le \frac{B}{2}\mathbb{E}_{\mu_s}\big[\|\nabla R\|^2\big] = 2B\mathbb{E}_{\mu_s}[R] \le 2BR_1,\]
using monotonicity together with Lemma~\ref{lem:ER} for the last step.
Hence $|f'(s)|\le L_f$ on $[\tau_1,\tau_Q]$, so $f$ is $L_f$-Lipschitz there. 
On panel $i$, with midpoint $c_i$, integration by parts gives
\[\int_{\tau_i}^{\tau_{i+1}} f(s)\,\mathrm{d}s 
- \frac{f(\tau_i)+f(\tau_{i+1})}{2}\,h_i 
= -\int_{\tau_i}^{\tau_{i+1}} (s - c_i)\,f'(s)\,\mathrm{d}s,\]
whose absolute value is at most $L_f\int_{\tau_i}^{\tau_{i+1}}|s-c_i|\,ds
= \frac{L_f}{4}h_i^2$. Summing,
\[|(A)| \le \frac{L_f\Delta}{4}(\tau_q - \tau_1) \le \frac{L_f\Delta}{4}(\tau_Q - \tau_1).\]
\textbf{Step 3 (term B).}
Fix $i$ and let
\[\delta_i := \frac{\mathrm{d} \mu_{\tau_i}^{K(\tau_i)}}
{\mathrm{d}\mu_{\tau_i}} - 1, \quad
\mathbb{E}_{\mu_{\tau_i}}[\delta_i] = 0, \quad
\mathbb{E}_{\mu_{\tau_i}}[\delta_i^2]
= \chi^2\big(\mu_{\tau_i}^{K(\tau_i)}\big\|\mu_{\tau_i}\big)
\leq \varepsilon_\star^2\]
by construction of $K(\cdot)$.
Since
\[\mathbb{E}_{\mu_{\tau_i}}[R^2]
= \mathbb{V}_{\mu_{\tau_i}}[R] + \mathbb{E}_{\mu_{\tau_i}}[R]^2
\leq 2BR_1 + R_1^2 = \widetilde R_2,\]
Cauchy--Schwarz gives
\[\big|\mathbb{E}_{\mu_{\tau_i}^{K(\tau_i)}}[R]
- \mathbb{E}_{\mu_{\tau_i}}[R]\big|
= \big|\mathbb{E}_{\mu_{\tau_i}}[R\,\delta_i]\big|
\leq \sqrt{\mathbb{E}_{\mu_{\tau_i}}[R^2]
\mathbb{E}_{\mu_{\tau_i}}[\delta_i^2]}
\leq \sqrt{\widetilde R_2}\,\varepsilon_\star.\]
Summing with the weights $w_i^{(q)}$, whose sum is $\tau_q - \tau_1$, we find
$|(B)| \le \tfrac12\sqrt{\widetilde R_2}\,\varepsilon_\star(\tau_Q-\tau_1)$.
 
\textbf{Step 4 (term C).}
For a given $i$, write $\bar R := \mathbb{E}_{\mu_{\tau_i}}[R]$. Since $\mathbb{E}_{\mu_{\tau_i}}[\delta_i]=0$,
subtracting the constant $\bar R$ before applying Cauchy--Schwarz does not change the identity,
and it minimizes the resulting bound:
\begin{align*}
\mathbb{E}_{\mu_{\tau_i}^{K(\tau_i)}} \big[(R-\bar R)^2\big]
&= \mathbb{E}_{\mu_{\tau_i}} \big[(R-\bar R)^2(1+\delta_i)\big] &&\\
&= \mathbb{V}_{\mu_{\tau_i}}[R]
+ \mathbb{E}_{\mu_{\tau_i}} \big[(R-\bar R)^2\delta_i\big] &&\\
&\le \mathbb{V}_{\mu_{\tau_i}}[R]
+ \sqrt{\mathbb{E}_{\mu_{\tau_i}}[(R-\bar R)^4]}\,\varepsilon_\star
\end{align*}
by Cauchy--Schwarz. Using the Poincaré bound
$\mathbb V_{\mu_{\tau_i}}[R]\le2BR_1$ and Lemma~\ref{lem:centered4th},
\[\mathbb{E}_{\mu_{\tau_i}^{K(\tau_i)}}\big[(R-\bar R)^2\big]
\leq 2BR_1 + \sqrt{\widehat R_4}\,\varepsilon_\star.\]
Since variance is the minimizer of
$c \mapsto \mathbb{E}_{\mu_{\tau_i}^{K(\tau_i)}}[(R-c)^2]$,
\[\mathbb{V}_{\mu_{\tau_i}^{K(\tau_i)}}[R]
\leq \mathbb{E}_{\mu_{\tau_i}^{K(\tau_i)}} \big[(R-\bar R)^2\big]
\leq 2BR_1 + \sqrt{\widehat R_4}\,\varepsilon_\star
=: \widehat V.\]
The draws $L_1^{(i)},\ldots,L_{M}^{(i)}$ are i.i.d.\ from $\mu_{\tau_i}^{K(\tau_i)}$ and, crucially, independent across $i$, since Algorithm~\ref{alg:ti} draws a fresh batch of $M$ chains at every grid point. Define 
\[X_i := \widehat R_{M}(\tau_i) - \mathbb{E}_{\mu_{\tau_i}^{K(\tau_i)}}[R] \quad \Rightarrow \quad
\mathbb{E}[X_i]=0, \, \mathbb{V}[X_i] \le \frac{\widehat V}{M}.\] 
What Step~0 actually needs is $\mathbb{E}[\max_q|(C)_q|]$, not $\max_q\mathbb{E}|(C)_q|$, 
so rather than summing panel-wise, we split the trapezoid sum into two martingales. Define
\[S_q := \sum_{i<q}h_iX_i, \quad S'_q := \sum_{i<q}h_iX_{i+1}, \quad\text{so that}\quad (C)_q = \frac12\sum_{i<q}h_i\,\frac{X_i+X_{i+1}}{2} = \frac{S_q+S'_q}{4}.\]
Let $\mathcal{F}_q := \sigma(X_1,\dots,X_q)$, with $\mathcal{F}_0$ trivial. $\{S'_q\}$ is a martingale with respect to $\{\mathcal{F}_q\}$, since its increment $h_q X_{q+1}$ is independent of $\mathcal{F}_q$ and has mean zero. $\{S_q\}$ is a martingale with respect to $\{\mathcal{F}_{q-1}\}$, since $S_q$ is $\mathcal{F}_{q-1}$-measurable, and its increment $h_q X_q$ is independent of $\mathcal{F}_{q-1}$ and has mean zero.
Applying Doob's $L^2$ maximal inequality to each,
\[\mathbb{E}\!\left[\max_q|S_q|\right] \leq 2\sqrt{\mathbb{V}[S_Q]}
= 2\sqrt{\textstyle\sum_{i<Q}h_i^2\,\mathbb{V}[X_i]}
\leq 2\sqrt{\frac{\widehat V \Delta (\tau_Q-\tau_1)}{M}}\]
and
\[\mathbb{E}\!\left[\max_q|S'_q|\right] \leq 2\sqrt{\sum_{i<Q}h_i^2\,\mathbb{V}[X_{i+1}]} \leq 2\sqrt{\frac{\widehat V \Delta (\tau_Q-\tau_1)}M},\]
because $h_i\le\Delta$ 
and $\sum_{i<Q}h_i=\tau_Q-\tau_1$. Since $\max_q|(C)_q| \leq \frac14(\max_q|S_q| + \max_q|S'_q|)$, we obtain
\[\mathbb{E}\big[\max_q|(C)_q|\big] 
\le \sqrt{\frac{\widehat V\,\Delta\,(\tau_Q-\tau_1)}{M}}.\]
\textbf{Step 5 (assemble).}
Terms $(A)$ and $(B)$ are bounded surely for every $q$ at once (Steps~2--3), so only $(C)$ needed Step~4's maximal inequality. Combining with Step~0,
\[
\mathbb{E}\big[W_1(\nu_\star^Q,\nu_\star^{Q,\mathrm{TI}})\big]
\le \frac{\tau_Q-\tau_1}{2}\left(\frac{\tau_Q-\tau_1}{2}
\left(\frac{L_f\Delta}{2}+\sqrt{\widetilde R_2}\,\varepsilon_\star\right)
+\sqrt{\frac{\widehat V\,\Delta\,(\tau_Q-\tau_1)}{M}}\right),
\]
which is the claimed bound.
\end{proof}

\begin{lemma}[Second-order quadrature error of the trapezoid rule]
\label{lem:TI-trap}
Let
\[L_f^{(2)} := \frac{N}{\tau_1^3} + \frac18\sqrt{2BR_1\widehat R_4}.\]
Then the bound of Lemma~\ref{lem:TI} holds with $L_f\Delta/2$ replaced by $\min\{L_f\Delta/2,\ L_f^{(2)}\Delta^2/6\}$.
\end{lemma}
\begin{proof}
Only term (A) in the proof of Lemma~\ref{lem:TI} changes. Write $\varphi(s) := \log Z_\mu(s)$, so that $f(s) = \frac{N}{2s} + \varphi'(s)$. Since $\{\mu_s\}$ is an exponential family with sufficient statistic $-R/2$, the derivatives of $\varphi$ are the cumulants of $-R/2$ under $\mu_s$: $\varphi'(s) = -\frac12\mathbb{E}_{\mu_s}[R]$, $\varphi''(s) = \frac14\mathbb{V}_{\mu_s}[R]$, as computed at the start of this subsection, and
\[\varphi'''(s) = -\frac18\,\mathbb{E}_{\mu_s}\big[(R-\mathbb{E}_{\mu_s}[R])^3\big].\]
Hence
\[f''(s) = \frac{N}{s^3} - \frac18\,\mathbb{E}_{\mu_s}\big[(R-\mathbb{E}_{\mu_s}[R])^3\big].\]
By Cauchy--Schwarz, the variance bound of Step~2 and Lemma~\ref{lem:centered4th},
\[\Big|\mathbb{E}_{\mu_s}\big[(R-\mathbb{E}_{\mu_s}[R])^3\big]\Big|
\leq \sqrt{\mathbb{V}_{\mu_s}[R]\;\mathbb{E}_{\mu_s}\big[(R-\mathbb{E}_{\mu_s}[R])^4\big]}
\leq \sqrt{2BR_1\widehat R_4},\]
so $|f''(s)| \leq L_f^{(2)}$ on $[\tau_1, \tau_Q]$. The error of the trapezoid rule on a panel of width $h_i$ is at most $\frac{1}{12}\|f''\|_\infty h_i^3$, so
\[|(A)| \leq \frac{L_f^{(2)}}{12}\sum_{i<q}h_i^3 \leq \frac{L_f^{(2)}\Delta^2}{12}(\tau_Q - \tau_1).\]
Combined with the bound $|(A)| \leq \frac{L_f\Delta}{4}(\tau_Q - \tau_1)$ of Step~2, this gives
\[|(A)| \leq \frac12\min\{L_f\Delta/2,\, L_f^{(2)}\Delta^2/6\}(\tau_Q-\tau_1),\] 
and Step~5 goes through unchanged.
\end{proof}
Lemma~\ref{lem:TI-trap} does not change Corollaries~\ref{cor:TV} and~\ref{cor:W1}. Their dependence on $Q$ is set by the discretization error of Lemma~\ref{lem:Q}, which is of order $1/Q$ whatever the quadrature rule (the first term of $\mathcal{W}_1$ in Table~\ref{tab:values}), and by the Monte Carlo term. The lemma is included to explain why the quadrature error of the trapezoid rule is much smaller in practice (Section~\ref{sec:numerical}).

\subsection{Regularity of the tilted family}
The tilted family $\mathrm{d}\mu_s \propto \exp(-\frac{s}{2}R(\bm{L}))\,\mathrm{d}\mu(\bm{L})$ is an exponential family with sufficient statistic $R$, so the Kullback--Leibler divergence between two of its members is a Bregman divergence of its log-partition function (\cite{Banerjee2005}).
Together with the variance bound of Step 2 in the proof of Lemma~\ref{lem:TI}, this makes the divergence quadratic, rather than linear, in the difference of the precisions.
\begin{lemma}\label{lem:bregman}
Let $\varphi(s) := \log Z_{\mu}(s)$. For all $s, t \geq 0$,
\[\mathrm{KL}(\mu_t\|\mu_s)
= \varphi(s) - \varphi(t) - \varphi'(t)(s-t)
\leq \frac{BR_1}{4}(s-t)^2.\]
\end{lemma}
\begin{proof}
Since
\[\log\frac{\mathrm{d}\mu_t}{\mathrm{d}\mu_s}(\bm{L})
= -\frac{t-s}{2}R(\bm{L}) + \varphi(s) - \varphi(t),\]
integrating against $\mu_t$ and using $\varphi'(t) = -\frac{1}{2}\mathbb{E}_{\mu_t}[R]$ from Appendix~\ref{sec:ti-fix} gives the identity.
Taylor's theorem with integral remainder then gives
\[\mathrm{KL}(\mu_t\|\mu_s) = \int_t^s (s-u)\,\varphi''(u)\,\mathrm{d}u
\leq \frac{(s-t)^2}{2}\sup_{u \geq 0}\varphi''(u),\]
and $\varphi''(u) = \frac{1}{4}\mathbb{V}_{\mu_u}[R] \leq \frac{BR_1}{2}$ by Step 2 in the proof of Lemma~\ref{lem:TI}.
\end{proof}

\subsection{Non-log-concavity, reparameterization, and importance sampling}
\begin{lemma}\label{lem:nonlc}
For any $\bm{Y}$ and $\bm{\Omega}$ with $N \geq 1$, the joint posterior density of $(\bm{L},\tau)$,
\[p(\bm{L},\tau \mid \bm{Y},\bm{\Omega}) \propto \tau^{N/2-1}\exp\left(-\frac{\tau}{2}R(\bm{L}) - \frac{1}{\lambda}\|\bm{L}\|_* - \frac{1}{B}\|\bm{L}\|_{\mathrm{F}}^2\right)\mathbbm{1}\{\tau_{\min} \leq \tau \leq \tau_{\max}\},\]
is not log-concave.
\end{lemma}
\begin{proof}
Let
\[f(\bm{L},\tau) := -\left(\frac{N}{2}-1\right)\log\tau + \frac{\tau}{2}R(\bm{L}) + \frac{1}{\lambda}\|\bm{L}\|_* + \frac{1}{B}\|\bm{L}\|_{\mathrm{F}}^2\]
be the negative log-density on $\mathbb{R}^{n_1\times n_2}\times(\tau_{\min},\tau_{\max})$.
Fix $\tau_0 \in (\tau_{\min},\tau_{\max})$, an observed entry $(i,j)$ with $\Omega_{ij} = 1$, and a matrix $\bm{L}_0$ of rank $n_1$ with $(\bm{L}_0)_{ij} \neq 0$, and let $\bm{E}_{ij}$ be the matrix with a one in entry $(i,j)$ and zeros elsewhere.
The nuclear norm $\|\bm{L}\|_* = \mathrm{tr}((\bm{L}\bm{L}^\top)^{1/2})$ is twice continuously differentiable on the open set of matrices of rank $n_1$, because the matrix square root is smooth on positive-definite matrices, and it is positively homogeneous of degree one, so its second directional derivative $h(c)$ at $c\bm{L}_0$ in the direction $\bm{E}_{ij}$ satisfies $h(c) = h(1)/c$ for $c > 0$.
The Hessian of $(u,v) \mapsto f(c\bm{L}_0 + u\bm{E}_{ij}, \tau_0 + v)$ at the origin is
\[\begin{bmatrix}
\tau_0 + \frac{2}{B} + \frac{h(1)}{\lambda c} & (c\bm{L}_0 - \bm{Y})_{ij}\\
(c\bm{L}_0 - \bm{Y})_{ij} & \frac{N/2-1}{\tau_0^2}
\end{bmatrix},\]
whose determinant is
\[\frac{N/2-1}{\tau_0^2}\left(\tau_0 + \frac{2}{B} + \frac{h(1)}{\lambda c}\right) - \left(c(\bm{L}_0)_{ij} - Y_{ij}\right)^2 \to -\infty
\quad\text{as } c \to \infty.\]
For large enough $c$, this Hessian therefore has a negative eigenvalue, so $f$ is not convex near $(c\bm{L}_0,\tau_0)$, and the posterior is not log-concave.
\end{proof}



\begin{remark}[The marginal posterior of $\tau$]
The marginal posterior of $\tau$ has density proportional to $\exp(\Lambda_\mu(\tau) - \log\tau)$ on $[\tau_{\min},\tau_{\max}]$, and by Appendix~\ref{sec:ti-fix} the second derivative of its logarithm is $-\frac{N/2-1}{\tau^2} + \frac14\mathbb{V}_{\mu_\tau}[R]$. At $\lambda = \infty$, every $\mu_\tau$ is Gaussian and
\[\mathbb{V}_{\mu_\tau}[R] = \frac{2N}{(c+\tau)^2}
+ \frac{4c^2\|\bm Y\circ\bm\Omega\|_{\mathrm{F}}^2}{(c+\tau)^3},
\qquad c = 2/B.\]
With $B = 2$ and $1 \in (\tau_{\min},\tau_{\max})$, the second derivative at $\tau = 1$ is $\frac18(\|\bm Y\circ\bm\Omega\|_{\mathrm{F}}^2 - 3N + 8)$, which is positive whenever $\|\bm Y\circ\bm\Omega\|_{\mathrm{F}}^2 > 3N - 8$. This is about three times the value under the standardization of Corollary~\ref{cor:TV}, so the example shows that the marginal can fail to be log-concave, not that it does in the regime of our corollaries.
\end{remark}

\begin{remark}[Reparameterization]
Two natural reparameterizations also fail to give a log-concave joint
posterior. 
(i)~In $(\bm L, s)$ with $s = \log\tau$, the Jeffreys prior is flat and the potential is 
\[-\frac N2 s + \frac{e^s}{2}R(\bm L) + \frac1\lambda\|\bm L\|_*
+ \frac1B\|\bm L\|_{\mathrm{F}}^2.\] 
Let
\[\bm D := \frac{\bm L - \bm Y)\circ\bm\Omega}{\sqrt{R(\bm L)}}\] 
and let $h \ge 0$ be the second directional derivative of $\|\cdot\|_*/\lambda$ at $\bm L$ in the direction $\bm D$. The Hessian restricted to the plane spanned by $(\bm D, 0)$ and $(\bm 0, 1)$ has
determinant
\[\frac{e^s R}{2}\left(e^s + \frac2B + h\right) - e^{2s}R
= \frac{e^s R}{2}\left(\frac2B + h - e^s\right).\]
Along the ray $c\bm L_0$ of the proof of Lemma~\ref{lem:nonlc}, $h = O(1/c)$ by homogeneity and $R(c\bm{L}_0) > 0$ for large $c$ since $(\bm{L}_0)_{ij} \neq 0$, so if $\tau_{\max} > 2/B$, the determinant is negative at any $s$ with $2/B < e^s < \tau_{\max}$ for sufficiently large $c$.
(ii)~With $\bm L' := \sqrt\tau\,\bm L$ and the Jacobian $\tau^{-n_1n_2/2}$, the potential of $(\bm L',\tau)$ is
\[\frac12\|(\bm L' - \sqrt\tau\,\bm Y)\circ\bm\Omega\|_{\mathrm{F}}^2
+ \frac{\|\bm L'\|_*}{\lambda\sqrt\tau} 
+ \frac{\|\bm L'\|_F^2}{B\tau} + a\log\tau,
\qquad a := \frac{n_1n_2 - N}{2} + 1 > 0.\]
At $\bm L' = \bm 0$, it equals $\frac\tau2\|\bm Y\circ\bm\Omega\|_{\mathrm{F}}^2 + a\log\tau$, which is strictly concave in $\tau$.
\end{remark}

\begin{lemma}[Importance sampling from the prior]\label{lem:is}
At $\lambda = \infty$, for every $\bm Y$, $\bm\Omega$ and $\tau > 0$,
the weights $w := e^{-\tau R(\bm L)/2}$ satisfy
\[\frac{\mathbb{E}_\mu[w^2]}{\mathbb{E}_\mu[w]^2}
\ \ge\ \left(\frac{1+\tau B/2}{\sqrt{1+\tau B}}\right)^{N}.\]
\end{lemma}
\begin{proof}
At $\lambda = \infty$ the entries of $\bm L$ are independent
$\mathcal N(0, B/2)$ under $\mu$, so $w$ factorizes over observed
entries into $e^{-\tau X^2/2}$ with $X \sim \mathcal N(m,\sigma^2)$,
$m = -Y_{ij}$ and $\sigma^2 = B/2$. Since
\[\mathbb{E}[e^{-aX^2}] = (1+2a\sigma^2)^{-1/2}
\exp\!\left(-\frac{am^2}{1+2a\sigma^2}\right),\] 
each factor of the ratio is
\[\frac{1+\tau\sigma^2}{\sqrt{1+2\tau\sigma^2}}
\exp\!\left(\tau m^2\!\left[\frac{1}{1+\tau\sigma^2}
- \frac{1}{1+2\tau\sigma^2}\right]\right)
\ge \frac{1+\tau\sigma^2}{\sqrt{1+2\tau\sigma^2}},\]
and the claim follows by taking the product over the $N$ observed entries.
\end{proof}
Since the effective sample size of $D$ weighted draws is about
$D\,\mathbb{E}_\mu[w]^2/\mathbb{E}_\mu[w^2]$, it decays exponentially
in $N$.

\FloatBarrier
\section{Proofs of the main results}\label{sec:proofs}

\begin{table}[t]
\centering
\small
\caption{Values of auxiliary quantities}
\label{tab:values}
\begin{tabular}{lp{6cm}r}
\toprule
Symbol & Value \\
\midrule
$\sigma(t)$ & $\frac{1}{4\sqrt{n_1n_2}\,\max\{\sqrt{n_1}/\lambda,\,\sqrt{2/B+t}\}}$ \\
$A(t)$ & $\frac{Bn_1n_2}{2}
+ \left(1 - \frac{2}{Bt+2} \right)^2 \|\bm{Y\circ\Omega}\|_{\mathrm F}^2$ \\
$u_0(t)$ & $\exp\!\left(\frac{\sqrt{n_1 A(t)}}{\lambda} 
+ \frac{Bn_1}{4\lambda^2}\right)-1$ \\
$R_1$ & $\|\bm{Y\circ\Omega}\|_{\mathrm{F}}^2 + \frac{NB}{2}$ \\
$R_2$ & $\left(\|\bm{Y \circ \Omega}\|_{\mathrm{F}}
+ \left(\frac{N(N+4)B^2}{4} \right)^{\frac{1}{4}} \right)^4$ \\
$\widetilde R_2$ & $R_1^2+2BR_1$ \\
$\widehat R_4$ & $512B^2R_1^2 + 3076B^4$ \\
$L_f$ & $\frac{N}{2\tau_1^2} + \frac{BR_1}{2}$ \\
$\widehat V$ & $2BR_1 + \sqrt{\widehat{R}_4}\,\varepsilon_\star$ \\
$\Delta$ & $\tau_Q (\kappa^{1/Q}-1)$ \\
$K(t)$ & $\begin{aligned}[t]
\ge{} & 2 + 32768\cdot\tfrac{Bn_1n_2}{C_\ell^2}
  e^{17/8}\max\!\left\{\tfrac{n_1}{\lambda^2},\tfrac{2}{B}+t\right\}\\
&\Big[4(\log\log(\bar u_0(t)/2)-\log\log 4)
+\tfrac{1}{\log 4}\log\!\left(\tfrac{8}{\varepsilon_\star^2}\right)\Big]
\end{aligned}$ \\
$\mathcal W_1$ & $\begin{aligned}[t]
&\tau_{\max} (\kappa^{1/Q}-1) 
+ \tfrac{(\tau_Q - \tau_1) \log^2 \kappa}{16Q^2} \\
&\quad 
\left( \left( \tfrac{N}{2} + \tfrac{\tau_{\max} R_1}{2}\right)^2
+ \tfrac{\tau_{\max} R_1}{6} + \tfrac{\tau_{\max}^2 R_2}{12} \right) \\
&\quad
+ \tfrac{(\tau_Q-\tau_1)^2}{4}
\left(\tfrac{L_f\Delta}{2} + \sqrt{\widetilde R_2}\,\varepsilon_\star\right) &&\\
&\quad
+ \tfrac{(\tau_Q-\tau_1)^{1.5}}{2}
\sqrt{\tfrac{\widehat{V} \Delta}{M}}
\end{aligned}$ \\
\bottomrule
\end{tabular}
\end{table}

\subsection{Proof of Theorem \ref{thm:TVguar}}
Let $\nu$ denote either $\mathbb{E}[\nu^{Q,\mathrm{TI}}_{\star}]$ or a realization of $\nu^{Q,\mathrm{TI}}_{\star}$.
By the triangle inequality,
\begin{align*}
&\mathrm{TV} \left(\int \mu_{\tau} \nu_{\star} (\mathrm{d}\tau),
\int \mu_{\tau}^{K(\tau)} \nu (\mathrm{d}\tau) \right) &&\\
&\leq \mathrm{TV} \left(\int \mu_{\tau} \nu_{\star} (\mathrm{d}\tau),
\int \mu_{\tau} \nu (\mathrm{d}\tau) \right) &&\\
&\qquad + \mathrm{TV} \left(
\int \mu_{\tau} \nu (\mathrm{d}\tau),
\int \mu_{\tau}^{K(\tau)} \nu (\mathrm{d}\tau) \right).
\end{align*}
Letting $\delta:=\frac{\mathrm{d}\mu_\tau^{K(\tau)}}{\mathrm{d}\mu_\tau}-1$, Cauchy--Schwarz gives
\[
\mathrm{TV}(\mu_\tau, \mu_\tau^{K(\tau)}) 
= \frac12\mathbb{E}_{\mu_\tau}|\delta|
\le \frac12\sqrt{\chi^2(\mu_\tau^{K(\tau)}\,\|\,\mu_\tau)} 
\le \frac{\varepsilon_\star}{2},\]
so by definition of total variation distance
\begin{align*}
&\mathrm{TV}\left(
\int \mu_{\tau} \nu (\mathrm{d}\tau),
\int \mu_{\tau}^{K(\tau)} \nu (\mathrm{d}\tau) \right) &&\\
&= \sup_{|f|\leq 1} \frac{1}{2} \left| \int
\left(\int f \mathrm{d}\mu_{\tau}
- \int f \mathrm{d}\mu_{\tau}^{K(\tau)} \right)
\nu (\mathrm{d}\tau) \right| &&\\
&\leq \sup_{|f|\leq 1} \int \frac{1}{2}
\left|\int f \mathrm{d}\mu_{\tau}
- \int f \mathrm{d}\mu_{\tau}^{K(\tau)} \right|
\nu (\mathrm{d}\tau) &&\\
&\leq \int \mathrm{TV}(\mu_{\tau}, \mu_{\tau}^{K(\tau)})
\nu (\mathrm{d}\tau) &&\\
&\leq \frac{1}{2} \varepsilon_{\star}.
\end{align*}
The other term requires some Lipschitz analysis.
By Pinsker's inequality and Lemma~\ref{lem:bregman}, for any $\tau$ and $\tau'$,
\[\mathrm{TV}(\mu_{\tau},\mu_{\tau'})
\leq \sqrt{\frac{1}{2}\mathrm{KL}(\mu_{\tau'}\|\mu_{\tau})}
\leq \sqrt{\frac{BR_1}{8}}\,|\tau-\tau'|.\]
By definition of total variation distance, for any $\gamma \in \Gamma(\nu_{\star}, \nu)$,
\begin{align*}
&\mathrm{TV}\left(\int \mu_{\tau} \nu_{\star} (\mathrm{d}\tau),
\int \mu_{\tau} \nu (\mathrm{d}\tau) \right) &&\\
&= \sup_{|f|\leq 1} \frac{1}{2} \left| \int
\left(\int f \mathrm{d}\mu_{\tau}
- \int f \mathrm{d}\mu_{\tau'} \right)
\gamma(\mathrm{d}\tau, \mathrm{d}\tau') \right| &&\\
&\leq \int \mathrm{TV}(\mu_{\tau},\mu_{\tau'})
\gamma(\mathrm{d}\tau, \mathrm{d}\tau')
\leq \sqrt{\frac{BR_1}{8}} \int |\tau-\tau'|
\gamma(\mathrm{d}\tau, \mathrm{d}\tau'),
\end{align*}
and taking the infimum over $\gamma$ gives
\[\mathrm{TV}\left(\int \mu_{\tau} \nu_{\star} (\mathrm{d}\tau),
\int \mu_{\tau} \nu (\mathrm{d}\tau) \right)
\leq \sqrt{\frac{BR_1}{8}}\, W_1(\nu_{\star}, \nu).\]
Now we bound the $W_1$ term. For $\nu = \mathbb{E}[\nu_{\star}^{Q,\mathrm{TI}}]$, we can see that
\[W_1(\nu_{\star}, \mathbb{E}[\nu_{\star}^{Q,\mathrm{TI}}])
\leq \mathbb{E}[W_1(\nu_{\star}, \nu_{\star}^{Q,\mathrm{TI}})] \]
by Jensen's inequality, and
\[\mathbb{E}[W_1(\nu_{\star},\nu_{\star}^{Q,\mathrm{TI}})]
\leq W_1(\nu_{\star},\nu^Q_{\star})
+ \mathbb{E}[W_1(\nu^{Q}_{\star},\nu^{Q,\mathrm{TI}}_{\star})]
\leq \mathcal{W}_1\]
by the triangle inequality and Lemmas \ref{lem:Q} and \ref{lem:TI}, which gives the first bound.
For $\nu = \nu_{\star}^{Q,\mathrm{TI}}$, Markov's inequality gives
\[W_1(\nu^{Q}_{\star},\nu^{Q,\mathrm{TI}}_{\star}) 
\leq \frac1p \mathbb{E}[W_1(\nu^{Q}_{\star},\nu^{Q,\mathrm{TI}}_{\star})]\]
with probability at least $1-p$, in which case
\[W_1(\nu_{\star},\nu_{\star}^{Q,\mathrm{TI}})
\leq W_1(\nu_{\star},\nu^Q_{\star})
+ \frac{\mathbb{E}[W_1(\nu^{Q}_{\star},\nu^{Q,\mathrm{TI}}_{\star})]}{p}
\leq \frac{\mathcal{W}_1}{p},\]
which gives the second bound.

\subsection{Proof of Theorem \ref{thm:W1guar}}
Let $\nu$ be as in the proof of Theorem~\ref{thm:TVguar}. By the triangle inequality,
\begin{align*}
&W_1 \left(\int \mu_{\tau} \nu_{\star} (\mathrm{d}\tau),
\int \mu_{\tau}^{K(\tau)} \nu (\mathrm{d}\tau) \right) &&\\
&\leq W_1 \left(\int \mu_{\tau} \nu_{\star} (\mathrm{d}\tau),
\int \mu_{\tau} \nu (\mathrm{d}\tau) \right) &&\\
&\qquad + W_1 \left(
\int \mu_{\tau} \nu (\mathrm{d}\tau),
\int \mu_{\tau}^{K(\tau)} \nu (\mathrm{d}\tau) \right).
\end{align*}
By Talagrand's $T_2$ inequality for the $2/B$-strongly log-concave $\mu_\tau$ and the definition of $K(\tau)$, we have
\[W_1^2(\mu_{\tau}, \mu_{\tau}^{K(\tau)})
\leq B \mathrm{KL}\!\left(\mu_{\tau}^{K(\tau)} \Big\| \mu_{\tau} \right)
\leq B \chi^2\!\left(\mu_{\tau}^{K(\tau)} \Big\| \mu_{\tau} \right)
\leq B\varepsilon_{\star}^2,\]
so by definition of Wasserstein distance
\begin{align*}
&W_1\left(
\int \mu_{\tau} \nu (\mathrm{d}\tau),
\int \mu_{\tau}^{K(\tau)} \nu (\mathrm{d}\tau) \right) &&\\
&= \sup_{\mathrm{Lip}(f)\leq 1} \left| 
\int \left(\int f \mathrm{d}\mu_{\tau}
- \int f \mathrm{d}\mu_{\tau}^{K(\tau)} \right)
\nu (\mathrm{d}\tau) \right| &&\\
&\leq \sup_{\mathrm{Lip}(f)\leq 1} 
\int \left|\int f \mathrm{d}\mu_{\tau}
- \int f \mathrm{d}\mu_{\tau}^{K(\tau)} \right|
\nu (\mathrm{d}\tau) &&\\
&\leq \int W_1(\mu_{\tau}, \mu_{\tau}^{K(\tau)})
\nu (\mathrm{d}\tau) &&\\
&\leq \sqrt{B}\,\varepsilon_{\star}. 
\end{align*}
The other term requires some Lipschitz analysis.
By Talagrand's inequality for the $2/B$-strongly log-concave measure $\mu_{\tau}$ and Lemma~\ref{lem:bregman}, for any $\tau$ and $\tau'$,
\[W_1(\mu_{\tau},\mu_{\tau'})
\leq \sqrt{B\,\mathrm{KL}(\mu_{\tau'}\|\mu_{\tau})}
\leq \frac{B\sqrt{R_1}}{2}\,|\tau-\tau'|.\]
By definition of Wasserstein distance, for any $\gamma \in \Gamma(\nu_{\star}, \nu)$,
\begin{align*}
&W_1\left(\int \mu_{\tau} \nu_{\star} (\mathrm{d}\tau),
\int \mu_{\tau} \nu (\mathrm{d}\tau) \right) &&\\
&= \sup_{\mathrm{Lip}(f)\leq 1} \left| \int
\left(\int f \mathrm{d}\mu_{\tau}
- \int f \mathrm{d}\mu_{\tau'} \right)
\gamma(\mathrm{d}\tau, \mathrm{d}\tau') \right| &&\\
&\leq \int W_1(\mu_{\tau},\mu_{\tau'})
\gamma(\mathrm{d}\tau, \mathrm{d}\tau')
\leq \frac{B\sqrt{R_1}}{2} \int |\tau-\tau'|
\gamma(\mathrm{d}\tau, \mathrm{d}\tau'),
\end{align*}
and taking the infimum over $\gamma$ gives
\[W_1\left(\int \mu_{\tau} \nu_{\star} (\mathrm{d}\tau),
\int \mu_{\tau} \nu (\mathrm{d}\tau) \right)
\leq \frac{B\sqrt{R_1}}{2}\, W_1(\nu_{\star}, \nu).\]
The bounds on $W_1(\nu_{\star}, \nu)$ are the same as in the proof of Theorem~\ref{thm:TVguar}.

\subsection{Proof of Corollary \ref{cor:TV}}

$\lambda = \Theta(n_1^a)$ and $B = \Theta(1)$ imply
\[\log(u_0(t)+1) = \Theta\!\left(n_1^{1/2-a} \sqrt{A(t)} + n_1^{1-2a} \right),\] 
and because $A(t) = \Theta(n_1 n_2)$ under our standardization regime, we have
\[\log \log (u_0(t)/2) 
\leq \log\!\left(O\!\left(n_1^{1-a} \sqrt{n_2} + n_1^{1-2a}\right)\right)
= O(\log(n_1 n_2)).\]
Moreover, 
\[\max\{n_1/\lambda^2, 2/B + t\} = \Theta(\max\{n_1^{1-2a}, 1\})\] 
uniformly over $t \in [\tau_{\min},\tau_{\max}]$, and $\log(1/\varepsilon_\star^2) = O(\log(N/\eta))$ for the choice of $\varepsilon_\star$ below.
It follows that 
\[K(t) = \widetilde{O}(n_1 n_2 \max\{n_1^{1-2a}, 1\}).\]
Under our standardization regime, for $\varepsilon_\star \leq 1$, we have 
\[R_1 = \Theta(N), \quad R_2 = O(N^2), \quad \widetilde R_2 = O(N^2), \quad \widehat R_4 = O(N^2),\] 
\[L_f = O(N), \quad \widehat V = O(N), \quad \Delta = O(1/Q), \quad \tau_Q - \tau_1 = O(1).\] 
Hence
\[\mathcal{W}_1 = O\!\left(\frac{N}{Q} + \frac{N^2}{Q^2} + N\varepsilon_\star + \sqrt{\frac{N}{QM}}\right),\]
and by Theorem~\ref{thm:TVguar}, the TV distance is at most
\[\frac{\varepsilon_\star}{2} + O\!\left(\frac{N^{3/2}}{Q} + \frac{N^{5/2}}{Q^2} + N^{3/2}\varepsilon_\star + \frac{N}{\sqrt{QM}}\right).\]
To get our overall TV distance down to $\eta < 1$, 
because of the fact that the $M$-term depends only on the product $QM$
while the other terms depend only on $Q$, we set $M = 1$. 
We then choose $\varepsilon_\star = c\,\eta N^{-3/2}$ and $Q = \lceil C N^2/\eta^2 \rceil$ for constants $c, C > 0$.
Since $\eta \leq 1 \leq N$, this gives 
\[N^{3/2}\varepsilon_\star = c\eta, \quad 
\frac{N}{\sqrt{Q}} \leq \frac{\eta}{\sqrt{C}}, \quad 
\frac{N^{3/2}}Q \leq \frac{\eta}C, \quad 
\frac{N^{5/2}}{Q^2} \leq \frac{\eta}{C^2},\] 
so the TV distance is at most $\eta$ for $c$ small enough and $C$ large enough.
It follows that the number of RWM steps needed is
\[QMK = \widetilde{O}\!\left( \frac{N^2 n_1 n_2 \max\{n_1^{1-2a}, 1\}}{\eta^2} \right).\]

\subsection{Proof of Corollary \ref{cor:W1}}

If $B = \Theta(1)$, the first bound of Theorem \ref{thm:W1guar} is the same in big-O terms as that of Theorem \ref{thm:TVguar}, so the proof of Corollary \ref{cor:TV} applies verbatim and gives the same number of RWM steps.
Each coordinate projection $\bm{L} \mapsto L_{ij}$ is $1$-Lipschitz with respect to the Frobenius norm, so the W1 distance between the corresponding entry marginals is at most the Frobenius W1 distance, which gives the second statement.

\section{Experimental details and additional results}
\label{app:experiments}

\subsection{Implementation}
All experiments use Python with NumPy and SciPy, with fixed random seeds. The code reproduces every number and figure in Section~\ref{sec:numerical} and this appendix. The grid is the log-spaced grid of Algorithm~\ref{alg:ti}, 
and chains start from $\mu_\tau^0$, with three exceptions: the random-walk arm of the exact benchmark and its run of Algorithm~\ref{alg:post-samp}, and the burn-in measurement of the scaling study, start from the untilted prior $\mu$ (at $\lambda = \infty$, $\mu_\tau^0 = \mu_\tau$, so starting there would measure nothing); and the chains for the panel-data example start from a dispersed warm start (observed entries from the Gaussian part of $\mu_\tau$, missing entries at a rank-3 imputation plus $\mathcal{N}(0, 0.3^2)$ noise). Where stated, the inner sampler is Algorithm~\ref{alg:mala}, with its step size and diagonal preconditioner adapted in discarded runs to an acceptance rate of about 57\% and then held fixed: in the coverage study they are adapted at five log-spaced values of $\tau$, each thermodynamic-integration chain uses the nearest, and they are re-adapted at the posterior-mean $\tau$ before Algorithm~\ref{alg:post-samp}.
The step size of Algorithm~\ref{alg:rwm} is $\sigma(\tau) = c(\tau)/\sqrt{d(2/B + \tau)}$, with $c(\tau)$ tuned by pilot runs, which are discarded, to an acceptance rate of about $25\%$. 
Algorithm~\ref{alg:ti} uses the trapezoid rule. As a variant, we also report the left-Riemann rule, which replaces the final loop of Algorithm~\ref{alg:ti} with
\[\widetilde\Lambda_q = \sum_{i=1}^{q-1} \widehat f(\tau_i)\,(\tau_{i+1} - \tau_i), \qquad \widehat f(\tau) := \frac{N}{2\tau} - \frac12\widehat R_M(\tau),\]
and needs chains only at $\tau_1, \dots, \tau_{Q-1}$. Lemma~\ref{lem:TI} holds for this variant with $L_f \Delta/2$ replaced by $L_f \Delta$, but its quadrature error is of first order in the grid spacing.
Distances between distributions on $\mathbb{R}$ are computed as $W_1(\rho, \rho') = \int |F_\rho - F_{\rho'}|$ from the CDFs, and convergence is monitored with split-$\widehat{R}$ across independent chains.
All experiments are implemented in Python (NumPy and SciPy) and run on one core of a 2-core cloud virtual machine. The full set takes about five hours, and the $\lambda = 0.2$ study of Section~\ref{sec:numerical} about 20 minutes.

\paragraph{The exact benchmark.} When $\lambda = \infty$, $\mu_\tau$ is Gaussian with independent entries, and the log marginal likelihood is, up to an additive constant,
\[\Lambda_\mu(\tau) = \frac{N}{2}\log\tau + \sum_{(i,j):\,\Omega_{ij} = 1}\left[-\frac{1}{2}\log\left(1 + \frac{\tau B}{2}\right) - \frac{\tau Y_{ij}^2}{2(1 + \tau B/2)}\right],\]
so $\nu_\star^Q$ and $\nu_\star$ are exact, and each entry's posterior marginal is a mixture over $\tau$ of Gaussians. 
Replicates with exact draws are fully independent. 
Replicates with random-walk chains are disjoint subsets of one batch of independently seeded chains and therefore also independent.

\subsection{Further results on the exact benchmark}
Figure~\ref{fig:supp-exact} shows the variability of $\nu_\star^{Q,\mathrm{TI}}$ across runs of Algorithm~\ref{alg:ti} and the accuracy of Algorithm~\ref{alg:post-samp} on individual entries. With $M = 1$, the 5--95\% band of the CDF of $\nu_\star^{Q,\mathrm{TI}}$ over 200 runs is wide; with $M = 16$ it is narrow and centered on $\nu_\star$. This is the gap between the averaged and the high-probability bounds of Theorems~\ref{thm:TVguar} and~\ref{thm:W1guar}.

\begin{figure}[h]
    \centering
    \includegraphics[width=0.85\textwidth]{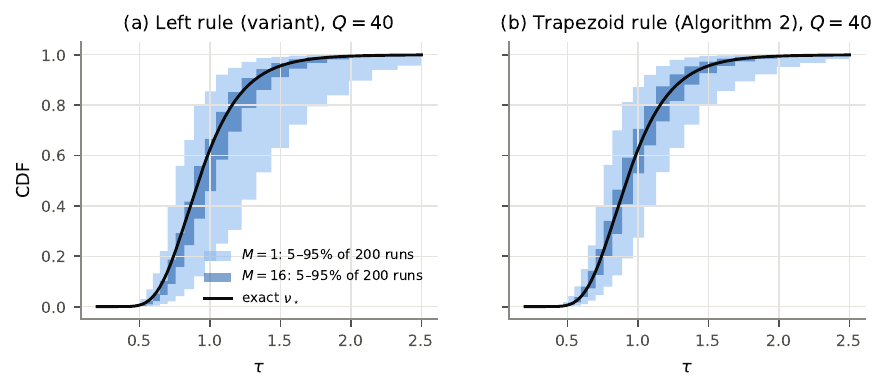}\\[4pt]
    \includegraphics[width=0.45\textwidth]{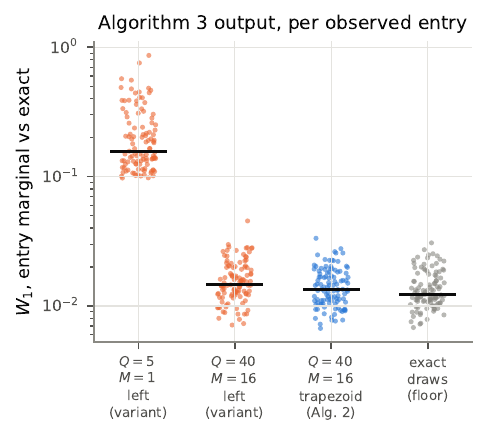}
    \caption{Exact benchmark. Top: 5--95\% bands of the CDF of $\nu_\star^{Q,\mathrm{TI}}$ over independent runs of Algorithm~\ref{alg:ti}. Bottom: per-entry $W_1$ distance between the output of Algorithm~\ref{alg:post-samp} and the exact posterior marginals, against the same distance for exact posterior draws.}
    \label{fig:supp-exact}
\end{figure}

\subsection{Posterior of \texorpdfstring{$\bm{L}$}{L} at \texorpdfstring{$\lambda = 0.2$}{lambda = 0.2}}
Figure~\ref{fig:supp-lambda02} compares the output of Algorithm~\ref{alg:post-samp} with the Gibbs reference of Section~\ref{sec:numerical} on individual entries and on the leading singular value. 
With $Q = 40$, $M = 16$ and the trapezoid rule, the two agree about as well as two halves of the Gibbs run that share no chains (median per-entry $W_1$ 0.024 against 0.024 on observed entries, 0.029 against 0.026 on missing ones; both comparisons use 1{,}000 draws against the same reference half). 
With $Q = 5$, $M = 1$ and the left-Riemann rule, the median per-entry $W_1$ distance on observed entries is 0.093,
and the posterior of the leading singular value is visibly shifted. The runs at $\lambda = 0.2$ use $K = 10{,}000$, below the 17{,}200--27{,}500 steps needed for split-$\widehat{R} < 1.05$ on $\|\bm{L}\|_*$. Results at $K = 2{,}500$ and $K = 10{,}000$ agree, and the output matches the long Gibbs reference.

\begin{figure}[h]
    \centering
    \includegraphics[width=0.85\textwidth]{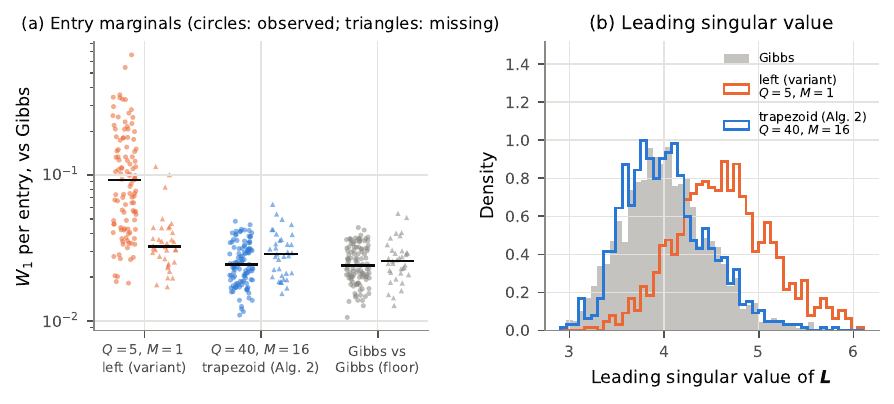}
    \caption{Algorithm~\ref{alg:post-samp} against the Gibbs reference at $\lambda = 0.2$.}
    \label{fig:supp-lambda02}
\end{figure}

\subsection{Estimator comparison, scaling, and the step count \texorpdfstring{$K$}{K}}

Both studies in this subsection use one family of instances: an $n \times n$ matrix of rank $2$ with singular values $(n, 2n/3)$, so that the entries stay at a fixed scale, noise precision $\tau = 1$, each entry observed independently with probability $0.7$, $B = 2$ and $[\tau_{\min}, \tau_{\max}] = [0.2, 5]$, with one data set per $n$. The toy problem of Section~5 is the case $n = 12$. The estimator comparison uses $n = 6$, $12$ and $32$ ($N = 25$, $109$ and $745$).

The estimators in Section~\ref{sec:numerical} all use the same 640 exact draws per replicate on the grid with $Q = 40$. Importance sampling reweights draws from the untilted prior $\mu$. Stepping-stone sampling estimates each ratio $Z(\tau_{i+1})/Z(\tau_i)$ by importance sampling from $\mu_{\tau_i}$ (\cite{Xie2011}), and geometric bridge sampling uses draws from both $\mu_{\tau_i}$ and $\mu_{\tau_{i+1}}$ with the geometric bridge (\cite{MengWong1996}). The top of Figure~\ref{fig:supp-est} shows the errors as $N$ grows and the resulting $W_1$ distance to $\nu_\star$.

For the scaling study, 
we use 
$n \in \{6, 8, 12, 16, 24\}$. For each $N$, we search over $(Q, M)$ for the smallest $QM$ at which the mean $W_1(\nu_\star^{Q,\mathrm{TI}}, \nu_\star)$ over replicates falls below a quarter of the posterior standard deviation of $\tau$, using 30 replicates with exact draws, and we multiply by an empirical burn-in $K$: the first step count at which the average of $R(\bm{L})$ over 128 random-walk chains started from $\mu$ is within 0.1 posterior standard deviations of its exact value at each of five test values of $\tau$. The search grid is coarse, so the fitted exponent of 1.74 is approximate (Figure~\ref{fig:supp-est}, bottom left).
The estimate is rough for three further reasons. The burn-in criterion is about one Monte Carlo standard error ($1/\sqrt{128} \approx 0.09$ standard deviations), so $K$ is noisy. At $\lambda = \infty$, $\mu_\tau^0 = \mu_\tau$, so $K$ is measured from the untilted prior, a start the algorithm never uses. And the target, relative $W_1$ on $\tau$, differs from the fixed accuracy of Corollary~\ref{cor:TV}.

For the step count $K$ (Figure~\ref{fig:supp-est}, bottom right), at $\tau \in \{0.3, 1, 4\}$ and $\lambda \in \{0.2, 1\}$, Table~\ref{tab:values} gives $K$ between $3 \times 10^{10}$ and $10^{12}$. Chains of 40{,}000 tuned steps reach split-$\widehat{R} < 1.05$ on $R(\bm{L})$ and $\|\bm{L}\|_*$ after 14{,}700--27{,}500 steps, but split-$\widehat{R} < 1.01$ is not reached within 40{,}000 steps (it is 1.02--1.03 there); the nuclear norm is the slowest statistic to converge. At $\lambda = \infty$, the prescribed $K$ is tiny but the starting $\chi^2$ divergence is infinite, so we do not report it.

\begin{figure}[h]
    \centering
    \includegraphics[width=\textwidth]{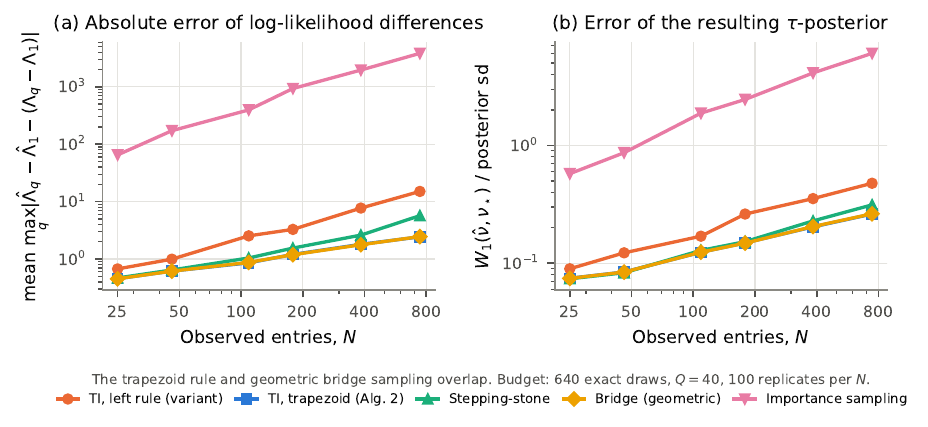}\\[4pt]
    \includegraphics[width=0.48\textwidth]{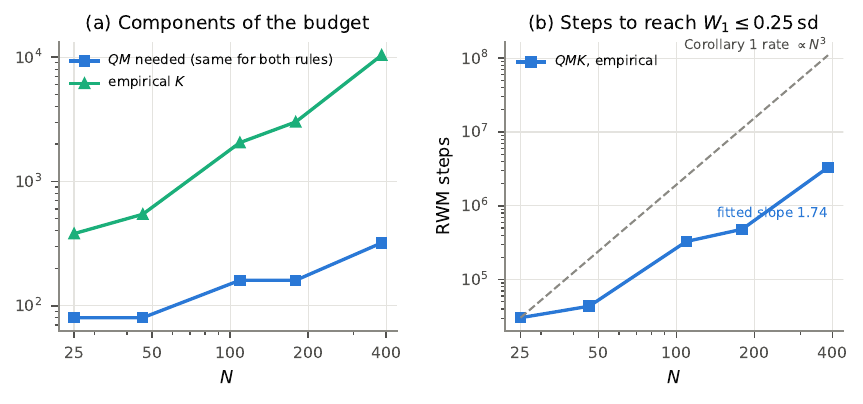}\hfill
    \includegraphics[width=0.48\textwidth]{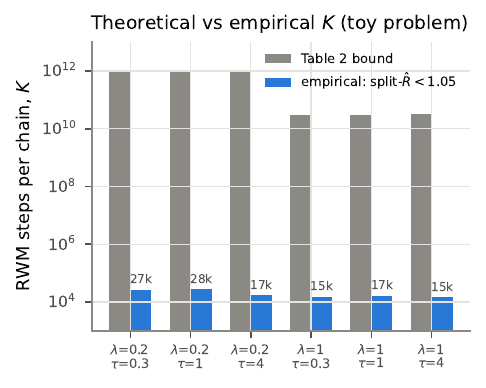}
    \caption{Top: estimator comparison at an equal budget of exact draws. Bottom left: total random-walk steps to reach a fixed accuracy, against $N$. Bottom right: step counts prescribed by Table~\ref{tab:values} against those needed in practice.}
    \label{fig:supp-est}
\end{figure}

\subsection{Coverage and sensitivity}
\paragraph{Recoverable instance.} The data are $12 \times 12$ matrices of rank 2 with singular values $(10, 6.7)$ and noise precision $\tau = 10$, with each entry observed independently with probability 0.3 or 0.7, $B = 2$ and $[\tau_{\min}, \tau_{\max}] = [1, 100]$, with 10 data sets per setting. We run Algorithms~\ref{alg:ti} and~\ref{alg:post-samp} with $Q = 30$, $M = 2$ and 200 posterior draws, and Algorithm~\ref{alg:mala} as the inner sampler with $K = 12{,}000$; random-walk Metropolis mixes too slowly here (split-$\widehat{R} = 1.08$ after 40{,}000 steps at $\lambda = 0.05$). Mixing is checked with 16 further chains at the posterior-mean $\tau$: split-$\widehat{R}$ on $R(\bm{L})$, $\|\bm{L}\|_*$ and the mean of the held-out entries is at most 1.03 in every run (median 1.01), and acceptance is at least 0.44 at every grid point. We omit $\lambda = \infty$: its prior is independent across entries, so the posterior of a held-out entry is its prior $\mathcal{N}(0, B/2)$. Figure~\ref{fig:supp-coverage} and the table below report averages over data sets: the coverage and mean width of 95\% credible intervals for held-out entries, and the held-out root-mean-square error of the posterior mean and of the posterior mode, each divided by that of predicting zero. The mode is computed at the true $\tau$, where $\lambda$ was calibrated, and at the posterior-mean $\tau$.
\begin{center}
\small
\begin{tabular}{lcccccc}
\toprule
 & & & & \multicolumn{2}{c}{Error (mode)} & \\
\cmidrule(lr){5-6}
$\lambda$ & Observed & Coverage & Error (mean) & true $\tau$ & post.\ $\tau$ & Mean $\tau$ \\
\midrule
0.05 & 30\% & 0.39 & 1.00 & 0.93 & 1.00 & 1.17 \\
0.05 & 70\% & 0.37 & 1.00 & 0.65 & 1.00 & 1.13 \\
0.5 & 30\% & 0.88 & 0.99 & 0.96 & 0.97 & 13.1 \\
0.5 & 70\% & 0.91 & 0.92 & 0.84 & 0.84 & 23.2 \\
\bottomrule
\end{tabular}
\end{center}
At $\lambda = 0.05$ the mode at the posterior-mean $\tau \approx 1.1$ is exactly zero, because the threshold $1/(\lambda\tau) \approx 18$ exceeds the largest true singular value. At the true $\tau$, it is informative. Intervals have width 0.83 at $\lambda = 0.05$ and about 3.0 at $\lambda = 0.5$.
We also tried making the prior on $\bm{L}$ scale with the noise, as in \cite{ParkCasella2008}, with density proportional to $\exp(-\sqrt{\tau}\|\bm{L}\|_*/\lambda' - \tau\|\bm{L}\|_{\mathrm F}^2/B)$. 
Empirically, with $\lambda' = 1/(2\sqrt{12})$ and $B = 50$, at 70\% observed, it raises coverage to 0.77, but the posterior mean of $\tau$ is still 1.3--1.6, against a true value of 10, and its posterior mean improves on predicting zero by only 6\%, against 36\% for its mode at the true $\tau$.

\begin{figure}[h]
    \centering
    \includegraphics[width=\textwidth]{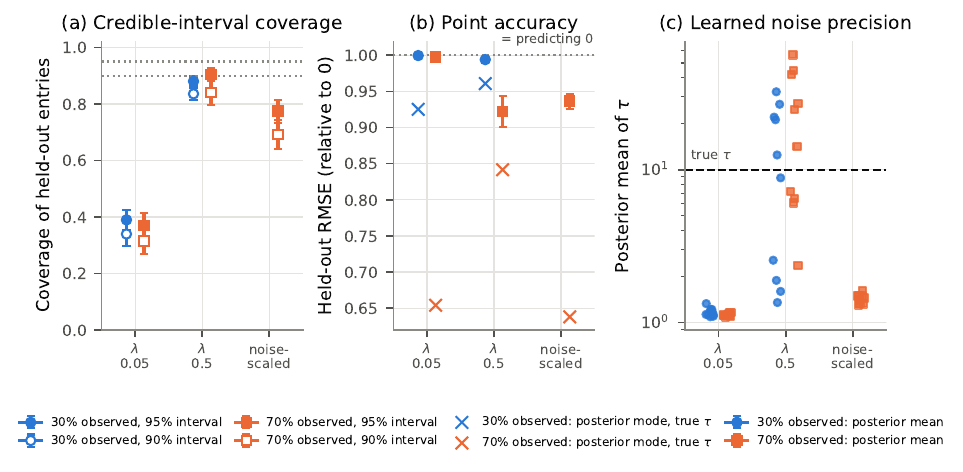}
    \caption{Recoverable instance: coverage of held-out entries, held-out error relative to predicting zero, and posterior mean of $\tau$ (true value 10), for two values of $\lambda$ and the noise-scaled prior.}
    \label{fig:supp-coverage}
\end{figure}

\paragraph{Sensitivity to $B$ and to the truncation interval.} On the toy problem with $\lambda = 0.2$, $B$ matters: the posterior mean of $\tau$ moves from 0.52 at $B = 0.5$ to 1.18 at $B = 32$, since a larger $B$ means less Frobenius shrinkage, and the width of held-out 90\% intervals grows from 1.2 to 2.9. Widening $[\tau_{\min}, \tau_{\max}]$ from $[0.2, 5]$ to $[0.01, 100]$ changes the posterior of $\tau$ very little (mean 0.57--0.60) while the posterior sits in the interior of the interval (Figure~\ref{fig:supp-sens}). 
The toy problem itself, whose singular values $(12, 8)$ are close to the noise level, is barely recoverable, so we use it only for the sampler comparisons above.

\begin{figure}[h]
    \centering
    \includegraphics[width=\textwidth]{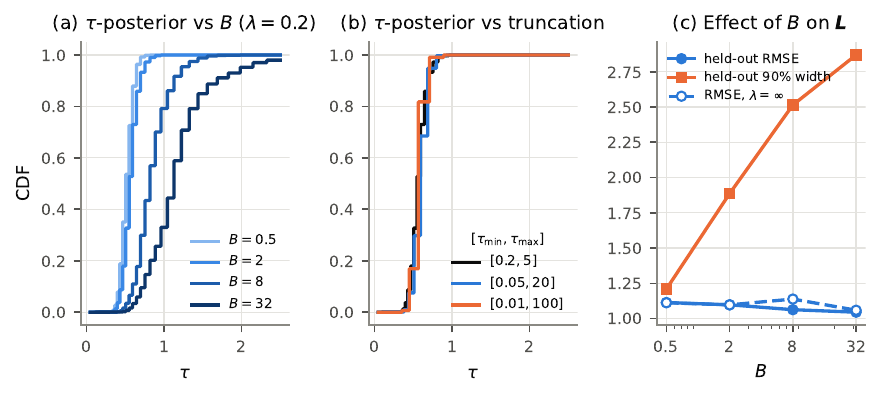}
    \caption{Toy problem, $\lambda = 0.2$: posterior of $\tau$ as $B$ and the truncation interval vary, and the effect of $B$ on held-out error and interval width.}
    \label{fig:supp-sens}
\end{figure}

\subsection{A larger example: random-walk Metropolis versus MALA}
On a problem of realistic size, random-walk Metropolis does not mix in a practical number of steps, while Algorithm~\ref{alg:mala} does. We use the per-capita cigarette-sales panel of \cite{Abadie2010} (39 states, 1970--2000), standardized, with California's 1989--2000 entries held out ($N = 1{,}197$), $\hat\tau = 46.4$ from the residuals of a rank-3 completion, $\lambda = 1/(\sqrt{\hat\tau}(\sqrt{31} + \sqrt{39}))$ (the regime $a = -1/2$), $[\tau_{\min}, \tau_{\max}] = [\hat\tau/100, 10\hat\tau]$ and $B = 2$. Thirty-two preconditioned random-walk chains give split-$\widehat{R}$ of 12.5, 7.3 and 5.0 on the mean of the held-out entries after 5{,}000, 10{,}000 and 20{,}000 steps, whereas Algorithm~\ref{alg:mala}, run for 40{,}000 steps within the trapezoid rule with $Q = 20$ and $M = 2$, reaches 1.012. The posterior itself shows the behavior described in Section~\ref{sec:numerical}: the posterior mean of $\tau$ is 4.4, a tenth of $\hat\tau$, and California's posterior counterfactual is nearly flat at its pre-1989 average, so we do not interpret it as an estimate of the program's effect.
At this $\lambda$ the posterior mode over-shrinks even at $\hat\tau$ (an average effect of $-45.9$ packs, against $-59.0$ at the posterior-mean $\tau$, $-19.5$ for synthetic control and $-11.5$ for rank-3 completion).

\end{document}